\documentclass[12pt]{article}
\usepackage[margin=1.25in]{geometry}
\usepackage[utf8]{inputenc}

\usepackage[english]{babel}
\usepackage[round]{natbib}
\usepackage[dvipsnames]{xcolor}

\usepackage{amsmath, amssymb, amsthm,mathtools,dsfont}
\usepackage{bbm}
\usepackage{natbib}
\usepackage{blkarray}
\usepackage{cancel}
\usepackage{enumitem}
\usepackage{euscript}
\usepackage{float}
\usepackage{bm}
\usepackage{graphicx}
\usepackage{mathrsfs}
\usepackage{microtype}
\usepackage{ragged2e}
\usepackage{sectsty}
\usepackage{thmtools}
\usepackage{tikz}
\usepackage[Symbolsmallscale]{upgreek}

\usepackage{microtype}
\usepackage{graphicx}
\usepackage{subcaption}

\usepackage{hyperref}

\newcommand{\cC}{\mathcal{C}}

\newcommand{\cF}{\mathcal{F}}

\newcommand{\cH}{\mathcal{H}}

\newcommand{\cM}{\mathcal{M}}
\newcommand{\cN}{\mathcal{N}}
\newcommand{\cO}{\mathcal{O}}
\newcommand{\cP}{\mathcal{P}}

\newcommand{\EE}{\mathbb{E}}

\newcommand{\kl}[2]{\mathrm{KL}(#1 \| #2)}

\newcommand*{\tvsq}[2]{\mathrm{{TV}}^2(#1, #2)}

\newcommand*{\triplenorm}[1]{{\left\vert\kern-0.25ex\left\vert\kern-0.25ex\left\vert #1
    \right\vert\kern-0.25ex\right\vert\kern-0.25ex\right\vert}}

\DeclareMathOperator{\supp}{supp}

\newcommand{\R}{\mathbb{R}}

\renewcommand{\phi}{\varphi}
\newcommand{\eps}{\varepsilon}
\newcommand{\sse}{\subseteq}

\newcommand*{\E}{\mathbb E}

\newcommand*{\defeq}{\coloneqq}
\newcommand*{\rd}{\mathrm{d}}
\newcommand*{\dd}{\, \rd}

\DeclareMathOperator*{\argmax}{argmax}

\DeclareMathOperator{\OTeps}{\mathrm{OT}_\eps}
\newcommand\msf[1]{\mathsf{#1}}

\usepackage{algorithm}
\usepackage{algpseudocode}

\usepackage{fancyhdr}

\usepackage[capitalize,noabbrev]{cleveref}

\usepackage{todonotes}
\theoremstyle{plain}
\newtheorem{theorem}{Theorem}[section]
\newtheorem{proposition}[theorem]{Proposition}

\newtheorem{lemma}[theorem]{Lemma}
\newtheorem{corollary}[theorem]{Corollary}
\theoremstyle{definition}

\theoremstyle{remark}
\newtheorem{remark}[theorem]{Remark}

\hypersetup{citecolor=purple, urlcolor=cyan, colorlinks=true, linkcolor=blue}

\title{Plug-in estimation of Schr\"odinger bridges
}
\author{Aram-Alexandre Pooladian\thanks{Center for Data Science, New York University. \tt aram-alexandre.pooladian@nyu.edu} \and Jonathan Niles-Weed\thanks{Center for Data Science and Courant Institute for Mathematical Science, New York University. \tt jnw@cims.nyu.edu}}

\begin{document}
\maketitle

\begin{abstract}
We propose a procedure for estimating the Schr{\"o}dinger bridge between two probability distributions. Unlike existing approaches, our method does not require iteratively simulating forward and backward diffusions or training neural networks to fit unknown drifts. Instead, we show that the potentials obtained from solving the static entropic optimal transport problem between the source and target samples can be modified to yield a natural plug-in estimator of the time-dependent drift that defines the bridge between two measures. Under minimal assumptions, we show that our proposal, which we call the \emph{Sinkhorn bridge}, provably estimates the Schr{\"o}dinger bridge with a rate of convergence that depends on the intrinsic dimensionality of the target measure. Our approach combines results from the areas of sampling, and theoretical and statistical entropic optimal transport. \looseness -1
\end{abstract}

\section{Introduction}
{Modern statistical learning tasks often involve not merely the comparison of two unknown probability distributions but also the estimation of \emph{transformations} from one distribution to another. Estimating such transformations is necessary when we want to generate new samples, infer trajectories, or track the evolution of particles in a dynamical system. In these applications, we want to know not only  how ``close" two distributions are, but also how to ``go" between them. }

Optimal transport theory defines objects that are well suited for both of these tasks \citep{Vil08,San15,CheNilRig24}. The $2$-Wasserstein distance is a popular tool for comparing probability distributions for data analysis in statistics \citep{carlier2016vector,chernozhukov2017monge,GhoSen22}, machine learning \citep{salimans2018improving}, and the applied sciences \citep{bunne2023learning,manole2022background}. Under suitable conditions, the two probability measures that we want to compare (say, $\mu$ and $\nu$)  induce an optimal transport map: the uniquely defined vector-valued function which acts as a transport map\footnote{$T$ is a transport map between $\mu$ and $\nu$ if given a sample $X \sim \mu$, its image under $T$ satisfies $T(X) \sim \nu$.} between $\mu$ and $\nu$ such that the distance traveled is minimal in the $L^2$ sense \citep{Bre91}. Despite being a central object in many applications, the optimal transport map is difficult to compute and suffers from poor statistical estimation guarantees in high dimensions; see \cite{hutter2021minimax, manole2021plugin,divol2022optimal}.

These drawbacks of the optimal transport map suggest that other approaches for defining a transport between two measures may often be more appropriate.
For example, \emph{flow based} or iterative approaches have recently begun to dominate in computational applications---these methods sacrifice the $L^2$-optimality of the optimal transport map to place greater emphasis on the tractability of the resulting transport. The work of \cite{chen2018neural} proposed continuous normalizing flows (CNFs), which use neural networks to model the vector field in an ordinary differential equation (ODE). This machinery was exploited by several groups simultaneously \citep{albergo2022building,lipman2022flow,liu2022flow} for the purpose of developing tractable constructions of vector fields that satisfy the \emph{continuity equation} (see \Cref{sec:dyn_conteq} for a definition), and whose flow maps therefore yield valid transports between source and target measures.

An increasingly popular alternative method for iterative transport is based on the Fokker--Planck equation (see \Cref{sec:dyn_fokplanck} for a definition). This formulation incorporates a diffusion term, and the resulting dynamics follow a \emph{stochastic} differential equation (SDE). 
Though there exist many stochastic dynamics that give rise to valid transports, a canonical role is played by the \emph{Schr{\"o}dinger bridge} (SB).
Just as the optimal transport map minimizes the $L^2$ distance in transporting between two distributions, the SB minimizes the \emph{relative entropy} of the diffusion process, and therefore has an interpretation as the ``simplest'' stochastic process bridging the two distributions---indeed, the SB originates as a \emph{Gedankenexperiment} (or ``thought experiment") of Erwin Schr{\"o}dinger in modeling the large deviations of diffusing gasses \citep{schrodinger1932theorie}. There are many equivalent formulations of the SB problem (see \cref{sec:background}), though for the purposes of transport, its most important property is that it gives rise to a pair of SDEs that interpolate between two measures $\mu$ and $\nu$:
\begin{align}
    \dd X_t = b_t^\star(X_t)\dd t + \sqrt{\eps}\dd B_t\,,& \quad X_0 \sim \mu, X_1 \sim \nu\,,\label{eq:forward_sde}\\
    \dd Y_t = d_t^\star(Y_t)\dd t + \sqrt{\eps}\dd B_t\,,& \quad Y_0 \sim \nu, Y_1 \sim \mu\,,\label{eq:backward_sde}
\end{align}
where $\eps > 0$ plays the role of thermal noise.\footnote{We assume throughout our work that the reference process is Brownian motion with volatility $\eps$; see \cref{sec:sb}.}
Concretely, \eqref{eq:forward_sde} indicates that samples from $\nu$ can be obtained by drawing samples from $\mu$ and simulating an SDE with drift $b^\star_t$, and \eqref{eq:backward_sde} shows how this process can be performed in reverse. Though these dynamics are of obvious use in generating samples, the difficulty lies in obtaining estimators for the drifts.

Nearly a century later, Schr{\"o}dinger's thought experiment has been brought to reality, having found applications in the generation of new images, protein structures, and more \citep{kawakita2022quantifying,liu2022deep,nusken2022bayesian,thornton2022riemannian,shi2022conditional,lee2024disco}. The foundation for these advances is the work of \cite{de2021diffusion},  who propose to train two neural networks to act as the forward and backward drifts, which are iteratively updated to ensure that each diffusion yields samples from the appropriate distribution. This is reminiscent of the iterative proportion fitting procedure of \citet{fortet1940resolution}, and can be interpreted as a version of Sinkhorn's matrix-scaling algorithm \citep{sinkhorn1964relationship,cuturi2013sinkhorn} on path space. 

{While the framework of \cite{de2021diffusion}  is popular from a computational perspective, it is worth emphasizing that this method is relatively costly, as it necessitates the undesirable task of simulating an SDE
at each training iteration. Moreover, despite the recent surge in applications, current methods do not come with statistical guarantees to quantify their performance.
In short, existing work leaves open the problem of developing tractable, statistically rigorous estimators for the Schr\"odinger bridge.}

\subsubsection*{Contributions}
We propose and analyze a computationally efficient estimator of the Schr{\"o}dinger bridge which we call the \emph{Sinkhorn Bridge}. Our main insight is that it is possible to estimate the \emph{time-dependent} drifts in~\eqref{eq:forward_sde} and~\eqref{eq:backward_sde} by solving a \emph{single, static} entropic optimal transport problem between samples from the source and target measures.
Our approach is to compute the potentials $(\hat{f}, \hat{g})$ obtained by running Sinkhorn's algorithm on the data $X_1,\ldots,X_m\sim\mu$ and $Y_1,\ldots,Y_n \sim \nu$ and plug these estimates into a simple formula for the drifts.
For example, in the forward case, our estimator reads
\begin{align*}
    \hat{b}_{t}(z) \defeq (1-t)^{-1}\Bigl(-z + \frac{\sum_{j=1}^n Y_{j} \exp\bigl((\hat{g}_j - \tfrac{1}{2(1-t)}\|z-Y_{j}\|^2)/\eps\bigr)}{\sum_{j=1}^n \exp\bigl((\hat{g}_j - \tfrac{1}{2(1-t)}\|z-Y_{j}\|^2)/\eps\bigr)} \Bigr)\,.
\end{align*}
See \cref{sec:schro_sink_back} for a detailed motivation for the choice of $\hat{b}_t$.
Once the estimated potential $\hat g$ is obtained from a single use of Sinkhorn's algorithm on the source and target data at the beginning of the procedure, computing $\hat b_t(z)$ for any $z \in \R^d$ and any $t \in (0, 1)$ is trivial.

We show that the solution to a discretized SDE implemented with the estimated drift $\hat{b}_{t}$ closely tracks the law of the solution to~\eqref{eq:forward_sde} on the whole interval $[0, \tau]$, for any $\tau \in [0, 1)$.
Indeed, writing $\msf P^\star_{[0, \tau]}$ for the law of the process solving~\eqref{eq:forward_sde} on $[0, \tau]$ and $\hat{\msf P}_{[0, \tau]}$ for the law of the process obtained by initializing from a fresh sample $X_0 \sim \mu$ and solving a discrete-time SDE with drift $\hat b_t$, we prove bounds on the risk
\begin{align*}
	\E [\tvsq{\hat{\msf P}_{[0, \tau]}}{\msf P^\star_{[0, \tau]}}]
\end{align*}
that imply that, for fixed $\eps > 0$ and $\tau \in [0, 1)$, the Schr\"odinger bridge can be estimated at the \emph{parametric} rate.
Moreover, though it is well known that such bounds must diverge as $\eps \to 0$ or $\tau \to 1$, we demonstrate that the rate of growth depends on the \emph{intrinsic} dimension $\msf k$ of the target measure rather than the ambient dimension $d$.
When $\msf k \ll d$, this gives strong justification for the use of the Sinkhorn Bridge estimator in high-dimensional problems.

To give a particular example in a special case, our results provide novel estimation rates for the \emph{F{\"o}llmer bridge}, an object which has also garnered interest in the machine learning community \citep{vargas2023bayesian,chen2024probabilistic,huang2024one}. In this setting, the source measure is a Dirac mass, and we suppose the target measure $\nu$ is supported on a ball of radius $R$ contained within a $\msf k$-dimensional smooth submanifold of $\R^d$. Taking the volatility level to be unity, we show that the F{\"o}llmer bridge up to time $\tau \in [0,1)$ can be estimated in total variation with precision $\epsilon_{\text{TV}}$ using $n$ samples and $N$ SDE-discretization steps, where
\begin{align*}
    n \asymp R^2(1-\tau)^{-\msf k - 2}\epsilon_{\text{TV}}^{-2}\,, \quad N \lesssim d R^4(1-\tau)^{-4}{\epsilon_{\mathrm{TV}}^{-2}}\,.
\end{align*}
As advertised, for fixed $\tau \in [0, 1)$, these bounds imply parametric scaling on the number of samples---which matches similar findings in the entropic optimal transport literature, see, e.g., \cite{stromme2023minimum}---and exhibit a ``curse of dimensionality'' only with respect to the \emph{intrinsic} dimension of the target, $\msf k$.
As our main theorem shows, these phenomena are not unique to the F\"ollmer bridge, and hold for arbitrary volatility levels and general source measures.
Moreover, by tuning $\tau$ appropriately, we show how these estimation results yield guarantees for sampling from the target measure $\nu$, see \cref{sec:follmer_sampling}.
These guarantees also suffer only from a ``curse of intrinsic dimensionality.''
Since the drifts arising from the F\"ollmer bridge can be viewed as the score of a kernel density estimator of $\nu$ with a Gaussian kernel (see~\eqref{eq:estdrift_follmer}), this benign dependence on the ambient dimension is a significant improvement over guarantees recently obtained for such estimators in the context of denoising diffusion probabilistic models~\citep{wibisono2024optimal}.
Our improved rates are due to the intimate connection between the SB problem and entropic optimal transport in which intrinsic dimensionality plays a crucial role \citep{groppe2023lower,stromme2023minimum}. We expound on this connection in the main text. 

We are not the first to notice the simple connection between the static entropic potentials and the SB drift.
\cite{finlay2020learning} first proposed to exploit this connection to simulate the SB by learning static potentials via a neural network-based implementation of Sinkhorn's algorithm; however, due to some notational inaccuracies and implementation errors, the resulting procedure was not scalable.
This work shows the theoretical soundness of their approach, with a much simpler, tractable algorithm and with rigorous statistical guarantees.

\paragraph*{Outline.}
\Cref{sec:background} contains the background information on both entropic optimal transport and the Schr{\"o}dinger bridge problem, and unifies the notation between these two problems. Our proposed estimator, the Sinkhorn bridge, is described in \Cref{sec:main}, and \Cref{sec:proof_sketch} contains our main results and proof sketches, with the technical details deferred to the appendix. Simulations are performed in \Cref{sec:numerics}.

\subsubsection*{Notation}
We denote the space of probability measures over $\R^d$ with finite second moment by $\cP_2(\R^d)$. We write $B(x,R) \sse \R^d$ to indicate the (Euclidean) ball of radius $R > 0$ centered at $x \in \R^d$. We denote the maximum of a and b by $a \vee b$. We write $a \lesssim b$ (resp. $a \asymp b$) if there exists
constants $C > 0$ (resp. $c, C > 0$ such that $a \leq Cb$ (resp. $cb \leq a \leq Cb$). 
We let $\msf{path} \defeq \cC([0,1],\R^d)$ be the space of paths with $X_t : \msf{path} \to \R^d$ given by the canonical mapping $X_t(h) = h_t$ for any $h \in \msf{path}$ and any $t \in [0,1]$. For a path measure $\msf P \in \cP(\msf{path})$ and any $t\in[0,1]$, we write $\msf P_t \defeq (X_t)_\sharp\msf P \in \cP(\R^d)$ for the $t^{\text{th}}$ marginal of $\msf P_t$. Similarly, for $s,t\in[0,1]$, we can define the joint probability measure $\msf P_{st} \defeq (X_s,X_t)_\sharp \msf P$. We write $\msf P_{[0,t]}$ for the restriction of the  $\msf P$ to $\cC([0,t],\R^d)$. Since $\msf{path}$ is a Polish space, we can define regular conditional probabilities for the law of a path given its value at time $t$, which we denote $\msf P_{|t}$. For any $s > 0$, we write $\Lambda_s := (2 \pi s)^{-d/2}$ for the normalizing constant of the density of the Gaussian distribution $\cN(0, s I)$.

\subsection{Related work}
\paragraph{On Schr{\"o}dinger bridges.} 
The connection between entropic optimal transport and the Schr{\"o}dinger bridge (SB) problem is well studied; see the comprehensive survey by \citet{leonard2013survey}. We were also inspired by the works of \citet{ripani2019convexity}, \citet{gentil2020entropic}, as well as  \cite{chen2016relation,
chen2021stochastic} (which cover these topics from the perspective of optimal control), and the more recent article by \citet{kato2024large} (which revisits the large-deviation perspective of this problem). The special case of the F{\"o}llmer bridge and its variants has been a topic of recent study in theoretical communities \citep{eldan2020stability,mikulincer2024brownian}.

Interest in computational methods for SBs has been explosive in over the last few years, see~\cite{de2021diffusion,shi2022conditional,bunne2023schrodinger,tong2023simulation,vargas2023bayesian,yim2023se, chen2024probabilistic,shi2024diffusion} for recent developments in deep learning. The works by~\cite{bernton2019schr,pavon2021data,vargas2021solving} use more traditional statistical methods to estimate the SB, with various goals in mind. For example \cite{bernton2019schr} propose a sampling scheme based on trajectory refinements using a approximate dynamic programming approach. \cite{pavon2021data} and \cite{vargas2021solving} propose methods to compute the (intermediate) density directly based on maximum likelihood-type estimators: \cite{pavon2021data} directly model the densities of interest and devise a scheme to update the weights; \cite{vargas2021solving} use Gaussian processes to model the forward and backward drifts, and update them via a maximum-likelihood type loss.

\paragraph*{On entropic optimal transport.} 
Our work is closely related to the growing literature on statistical entropic optimal transport, specifically on the developments surrounding the \emph{entropic transport map}. This object was introduced by ~\citet{pooladian2021entropic} as a computationally friendly estimator for optimal transport maps in the regime $\eps \to 0$; see also~\cite{pooladian2023minimax} for minimax estimation rates in the semi-discrete regime. When $\eps$ is fixed, the theoretical properties of the entropic maps have been analyzed~\citep{chiarini2022gradient,conforti2022weak,chewi2023entropic,conforti2023quantitative,divol2024tight} as well as their statistical properties~\citep{del2022improved,goldfeld2022statistical,goldfeld2022limit,gonzalez2022weak,rigollet2022sample,werenski2023estimation}.~\cite{nutz2021entropic,ghosal2021stability} study the stability of entropic potentials and plans in a qualitative sense under minimal regularity assumptions. Most recently,~\cite{stromme2023minimum} and~\cite{groppe2023lower} established the connections between statistical entropic optimal transport and intrinsic dimensionality (for both maps and costs).~\cite{daniels2021score}  investigates sampling using  entropic optimal transport couplings combined with neural networks. Closely related are the works by~\cite{chizat2022trajectory} and~\cite{lavenant2021towards}, which highlight the use of entropic optimal transport for trajectory inference.
A more flexible alternative to the entropic transport map was recently developed by~\cite{kassraie2024progressive}, who proposed a transport that progressively displaces the source measure to the target measure by computing a new entropic transport map at each step to approximate the~\emph{McCann interpolation} \citep{mccann1997convexity}.\looseness-1

\section{Background}\label{sec:background}
\subsection{Entropic optimal transport}\label{sec:eot}
\subsubsection{Static formulation}
Let $\mu,\nu \in \cP_{2}(\R^d)$ and fix $\eps > 0$. The \emph{entropic optimal transport} problem between $\mu$ and $\nu$ is written as 
\begin{align}\label{eq:eot}
    \OTeps(\mu,\nu) \defeq \inf_{\pi \in \Pi(\mu,\nu)} \iint \tfrac12\|x-y\|^2 \dd \pi(x,y) + \eps \kl{\pi}{\mu\otimes \nu}\,, 
\end{align}
where $\Pi(\mu,\nu) \sse \cP_2(\R^d \times \R^d)$ is the set of joint measures with left-marginal $\mu$ and right-marginal $\nu$, called the set of \emph{plans} or couplings, and where we define the Kullback--Leibler divergence as
\begin{align*}
    \kl{\pi}{\mu\otimes\nu} \defeq \int \log\Bigl(\frac{\dd\pi(x,y)}{\dd\mu(x)\dd\nu(y)}\Bigr)\dd\pi(x,y)\,,
\end{align*}
whenever $\pi$ admits a density with respect to $\mu\otimes\nu$, and ${+\infty}$ otherwise. Note that when $\eps = 0$, \eqref{eq:eot} reduces to the $2$-Wasserstein distance between $\mu$ and $\nu$ (see, e.g., \cite{Vil08,San15}).  The entropic optimal transport problem was introduced to the machine learning community by \cite{cuturi2013sinkhorn} as a numerical scheme for approximating the $2$-Wasserstein distance on the basis of samples.

\Cref{eq:eot} is a strictly convex problem, and thus admits a unique minimizer, called the \emph{optimal entropic plan}, written $\pi^\star \in \Pi(\mu,\nu)$.\footnote{Though $\pi^\star$ and the other objects discussed in this section depend on $\eps$, we will omit this dependence for the sake of readability, though we will track the dependence on $\eps$ in our bounds.} Moreover, a dual formulation also exists (see \cite{genevay2019entropy})
\begin{align}\label{eq:eot_dual}
\begin{split}
    \OTeps(\mu,\nu) &= \sup_{(f,g)\in \cF} \Phi^{\mu\nu}(f,g)\\
\end{split}
\end{align}
where $\cF = L^1(\mu)\times L^1(\nu)$ and
\begin{align}
        \Phi^{\mu\nu}(f,g) \defeq \int f \dd \mu + \int g \dd\nu - \eps \iint \Bigl(\Lambda_\eps e^{(f(x)+g(y)-\tfrac12\|x-y\|^2)/\eps} - 1\Bigr)\dd\mu(x)\dd\nu(y)\,,
\end{align}
where we recall $\Lambda_\eps = (2\pi\eps)^{-d/2}$. Solutions are guaranteed to exist when $\mu, \nu \in \cP_2(\R^d)$, and we call the dual optimizers the \emph{optimal entropic (Kantorovich) potentials}, written $(f^\star,g^\star)$.~\cite{Csi75} showed that the primal and dual optima are intimately connected through the following relationship:\footnote{The normalization factor $\Lambda_\eps$ is not typically used in the computational optimal transport literature, but it simplifies some formulas in what follows. Since the procedure we propose is invariant under translation of the optimal entropic potentials, this normalization factor does not affect either our algorithm or its analysis.} 
\begin{align}\label{eq:primal_dual}
    \dd\pi^\star(x,y) &= \Lambda_\eps \exp\Bigl(\frac{f^\star(x)+g^\star(y)-\tfrac12\|x-y\|^2}{\eps}\Bigr)\dd\mu(x)\dd\nu(y)\,.
\end{align}
Though $f^\star$ and $g^\star$ are \emph{a priori} defined almost everywhere on the support of $\mu$ and $\nu$, they can be extended to all of $\R^d$ (see \cite{mena2019statistical,nutz2021entropic}) via the optimality conditions
\begin{subequations}
\begin{align*}
& f^\star(x) = -\eps \log \Big(\Lambda_\eps \int e^{(g^\star(y) - \|x-y\|^2/2)/\eps}\dd \nu(y)\Big) \,, \\
& g^\star(y) = -\eps \log \Big( \Lambda_\eps \int e^{(f^\star(x) - \|x-y\|^2/2)/\eps}\dd \mu(x)\Big)\,.
\end{align*}
\end{subequations}
At times, it will be convenient to work with \emph{entropic Brenier potentials}, defined as 
\begin{align*}
    (\phi^\star,\psi^\star) \defeq (\tfrac12\|\cdot\|^2 - f^\star, \tfrac12\|\cdot\|^2 - g^\star)\,.
\end{align*}
Note that the gradients of the entropic Brenier potentials\footnote{Passing the gradient under the integral is permitted via dominated convergence under suitable tail conditions on $\mu$ and $\nu$.} are related to barycentric projections of the optimal entropic coupling
\begin{align*}
    \nabla\phi^\star(x) = \E_{\pi^\star}[Y|X=x]\,, \qquad \nabla\psi^\star(y) = \E_{\pi^\star}[X|Y=y]\,.
\end{align*}
See \citet[Proposition 2]{pooladian2021entropic} for a proof of this fact. 
By analogy with the unregularized optimal transport problem, these are called \emph{entropic Brenier maps}.
The following relationships can also be readily verified (see \citet[Lemma 1]{chewi2023entropic}):
\begin{align}\label{eq:enthessians}
    \nabla^2\phi^\star(x) = \eps^{-1}\text{Cov}_{\pi^\star}[Y|X=x]\,, \qquad \nabla^2\psi^\star(y) = \eps^{-1}\text{Cov}_{\pi^\star}[X|Y=y]\,.
\end{align}

\subsubsection{A dynamic formulation via the continuity equation}\label{sec:dyn_conteq}
Entropic optimal transport can also be understood from a dynamical perspective. Let $(\msf p_t)_{t \in [0,1]}$ be a family of measures in $\cP_2(\R^d)$, and let $(v_t)_{t\in[0,1]}$ be a family of vector fields. We say that the pair satisfies the \emph{continuity equation}, written $(\msf p_t,v_t) \in \mathfrak{C}$, if $\msf p_0 = \mu$, $\msf p_1 = \nu$, and, for $t \in [0,1]$,
\begin{align}\label{eq:conteq}
\partial_t \msf p_t + \nabla \cdot (v_t\msf p_t) = 0\,.
\end{align}
Solutions to \eqref{eq:conteq} are understood to hold in the weak sense (that is, with respect to suitably smooth test functions).

The continuity equation can be viewed as the analogue of the marginal constraints being satisfied (i.e., the set $\Pi(\mu,\nu)$ above): it represents both the conservation of mass and the requisite end-point constraints for the path $(\msf p_t)_{t \in [0,1]}$. With this, we can cite a clean expression of the dynamic formulation of \eqref{eq:eot} (see \citet{conforti2021formula} or \citet{chizat2020faster}) if $\mu$ and $\nu$ are absolutely continuous and have finite entropy:
\begin{align}\label{eq:dyneot}
    \OTeps(\mu,\nu) + C_0(\eps,\mu,\nu) = \!\!\inf_{(\msf p_t,v_t) \in \mathfrak{C}} \int_0^1\!\! \int \bigl(\frac{1}{2}\|v_t(x)\|^2 + \frac{\eps^2}{8}\|\nabla\log\msf p_t(x)\|^2 \bigr) \dd \msf p_t(x) \dd t\,,
\end{align}
where $ C_0(\eps,\mu,\nu) \defeq \eps\log(\Lambda_\eps) + \tfrac{\eps}{2}(H(\mu) + H(\nu))$ is an additive constant, with $H(\mu) \defeq \int \log(\!\dd\mu)\dd\mu$, similarly for $H(\nu)$.

The case $\eps = 0$ reduces to the celebrated Benamou--Brenier formulation of optimal transport~\citep{benamou2000computational}.
\subsubsection{A stochastic formulation via the Fokker--Planck equation}\label{sec:dyn_fokplanck}
Yet another formulation of the dynamic problem exists, this time based on the \emph{Fokker--Planck equation}, which is said to be satisfied by a pair $(\msf p_t,b_t) \in \mathfrak{F}$ if $\msf p_0 = \mu$, $\msf p_1 = \nu$, and, for $t \in [0,1]$,
\begin{align*}
    \partial_t \msf p_t + \nabla \cdot (b_t \msf p_t) = \frac{\eps}{2}\Delta \msf p_t\,.
\end{align*}
Then, under the same conditions as above,
\begin{align}\label{eq:dyneot2}
    \OTeps(\mu,\nu) + C_1(\eps,\mu) = \inf_{(\msf p_t,b_t)} \int_0^1 \int \frac{1}{2}\|b_t(x)\|^2 \dd \msf p_t(x) \dd t\,,
\end{align}
where $C_1(\eps,\mu) = \eps\log(\Lambda_\eps) + \eps H(\mu)$. The equivalence between the objective functions \eqref{eq:dyneot} and \eqref{eq:dyneot2}, as well as the continuity equation and Fokker--Planck equations, is classical. 
For completeness,  we provide details of these computations in \cref{app:twoformulations}. 
A key property of this equivalence is the following relationship which relates the optimizers of \eqref{eq:dyneot}, written $(\msf p_t^\star,v_t^\star)$ and \eqref{eq:dyneot2}, written $(\msf p^\star_t,b_t^\star)$:
\begin{align*}
    b_t^\star = v_t^\star + \frac{\eps}{2}\nabla \log \msf p_t^\star\,.
\end{align*}
We stress that the minimizer $\msf p_t^\star$ is the same for both \eqref{eq:dyneot} and \eqref{eq:dyneot2}.

\subsection{The Schr{\"o}dinger Bridge problem}\label{sec:sb}
We will now briefly develop the required machinery to understand the Schr{\"o}dinger bridge problem. We will largely following the expositions of  \cite{leonard2012schrodinger,leonard2013survey,ripani2019convexity,gentil2020entropic}. 

For $\eps > 0$, we let $\msf R \in \cP(\msf{path})$ denote the law of the reversible Brownian motion on $\R^d$ with volatility $\eps$, with the Lebesgue measure as the initial distribution.\footnote{The problem below remains well posed even though $\msf R$ is not a probability measure; see \citet{leonard2013survey} for complete discussions.} We write the joint distribution of the initial and final positions under $\msf R$ by $\msf R_{01}({\rm d}x, {\rm d}y) = \Lambda_\eps \exp(-\tfrac12\|x-y\|^2/\eps) \dd x \dd y$.

With the above, we arrive at Schr{\"o}dinger's bridge problem over path measures:
\begin{align}\label{eq:sb_main}
    \min_{\msf P \in \cP(\msf{path})} \eps \kl{\msf P}{\msf R} \quad \text{s.t.} \quad \msf P_0 = \mu\,, \msf P_1 = \nu\,,
\end{align}
where $\mu,\nu \in \cP_{2}(\R^d)$ and are absolutely continuous with finite entropy. Let $\msf P^\star$ be the unique solution to \eqref{eq:sb_main}, which exists as the problem is strictly convex. \cite{leonard2013survey} shows that there exist two non-negative functions $\mathfrak{f}^\star,\mathfrak{g}^\star : \R^d \to \R_+$ such that
\begin{align}\label{eq:Peps}
    \msf P^\star = \mathfrak {f}^\star(X_0) \mathfrak {g}^\star(X_1)\msf R\,,
\end{align}
where $\text{Law}(X_0) = \mu$ and $\text{Law}(X_1) = \nu$. 

A further connection can be made:  if we apply the chain-rule for the KL divergence by conditioning on times $t=0,1$, the objective function \eqref{eq:sb_main} decomposes into
\begin{align*}
    \eps \kl{\msf P}{\msf R} = \eps\kl{\msf P_{01}}{\msf R_{01}} + \eps\E_{\msf P}\kl{\msf P_{|01}}{\msf R_{|01}}\,.
\end{align*}
Under the assumption that $\mu$ and $\nu$ have finite entropy, it can be shown that the first term on the right-hand side is equivalent to the objective for the entropic optimal transport problem in~\eqref{eq:eot_dual}.
Moreover, the second term vanishes if we choose the measure $\msf P$ so that the conditional measure $\msf P_{|01}$ is the same as $\msf R_{|01}$, i.e., a Brownian bridge.
Therefore, the objective function in \eqref{eq:sb_main} is minimized when $\msf P_{01}^\star = \pi^\star$ and when $\msf P$ writes as a mixture of Brownian bridges with the distribution of initial and final points given by $\pi^\star$:
\begin{align}\label{eq:sb_first}
    \msf P^\star = \iint \msf R(\cdot|X_0=x_0, X_1=x_1) \pi^\star({\rm d}x_0,{\rm d}x_1)\,.
\end{align}

Much of the discussion above assumed that $\mu$ and $\nu$ are absolutely continuous with finite entropy; indeed, the manipulations in this section as well as in \cref{sec:dyn_conteq,sec:dyn_fokplanck} are not justified if this condition fails.
Though the finite entropy conditioned is adopted liberally in the literature on Schr\"odinger bridges, in this work we will have to consider bridges between measures that may not be absolutely continuous (for example, empirical measures).
Noting that the entropic optimal transport problem~\eqref{eq:eot} has a unique solution for \emph{any} $\mu, \nu \in \cP_2(\R^d)$, we leverage this fact to use~\eqref{eq:sb_first} as the \emph{definition} of the Schr\"odinger bridge between two probability measures: for any pair of probability distributions $\mu, \nu \in \cP_2(\R^d)$, their \emph{Schr\"odinger bridge} is the mixture of Brownian bridges given by~\eqref{eq:sb_first}, where $\pi^\star$ is the solution to the entropic optimal transport problem~\eqref{eq:eot}.
\section{Proposed estimator: The Sinkhorn bridge}\label{sec:main}
Our goal is to efficiently estimate the Schr{\"o}dinger bridge (SB) on the basis of samples. Let $\msf P^\star$ denote the SB between $\mu$ and $\nu$, and define the 
the time-marginal flow of the bridge by
\begin{align}\label{eq:ent_int}
    \msf p_{t}^\star \defeq \msf P_t^\star\,,\qquad  t\in[0,1]\,.
\end{align}
This choice of notation is deliberate: when $\mu$ and $\nu$ have finite entropy, the $t$-marginals of $\msf P^\star$ for $t \in [0, 1]$ solve the dynamic formulations~\eqref{eq:dyneot} and~\eqref{eq:dyneot2}~\citep[Proposition 4.1]{leonard2013survey}.
In the existing literature, $\msf p_{t}^\star$ is sometimes called the the \emph{entropic interpolation} between $\mu$ and $\nu$. See~\citet{leonard2012schrodinger,leonard2013survey,ripani2019convexity,gentil2020entropic} for interesting properties of entropic interpolations (for example, their relation to functional inequalities).
Our goal is to provide an estimator $\hat{\msf P}$ such that $\E[\tvsq{\hat{\msf P}_{[0, \tau]}}{\msf P^\star_{[0,\tau]}}]$ is small for all $\tau < 1$.
In particular, this marginals of our estimator $\hat{\msf P}$ are estimators $\hat{\msf p}_t$ of $\msf p_{t}^\star$ for all $t \in [0, 1)$.\footnote{For reasons that will be apparent in the next section, time $\tau=1$ must be excluded from the analysis.}

We call our estimator the \emph{Sinkhorn bridge}, and we outline its construction below. Our main observation involves revisiting some finer properties of entropic interpolations as a function of the static entropic potentials. 
Once everything is concretely expressed, a natural plug-in estimator will arise which is amenable to both computational and statistical considerations.

\subsection{From Schr{\"o}dinger to Sinkhorn and back}\label{sec:schro_sink_back}
We outline two crucial observations from which our estimator naturally arises. First, we note that $\msf p_{t}^\star$ can be explicitly expressed as the following density \citep[Theorem 3.4]{leonard2013survey}
\begin{align}\label{eq:marginal-density}
\begin{split}
    \msf p_{t}^\star({\rm d}z) &\defeq \cH_{(1-t)\eps}[\exp(g^\star/\eps)\nu](z)\cH_{t\eps}[\exp(f^\star/\eps) \mu](z)\dd z\,,
\end{split}
\end{align}
where $\cH_{s}$ is the \emph{heat semigroup}, which acts on a measure $Q$ via
\begin{align*}
    Q \mapsto \cH_{s}[Q](z) \defeq \Lambda_s\int e^{-\tfrac{1}{2s}\|x-z\|^2}{Q({\rm d}x)}\,.
\end{align*}
This expression for the marginal of distribution $\msf p_t^\star$ follows directly from \eqref{eq:sb_first}:
\begin{align*}
	\msf p_t^\star(z) & := {\iint \msf R_t(z|X_0=x_0,X_1=x_1)\pi^\star({\rm d}x_0,{\rm d}x_1)} \\
	& = \iint \cN(z|ty + (1-t)x, t(1-t)\eps)\pi^\star({\rm d}x,{\rm d}y) \\
	& = \Lambda_\eps \iint e^{((f^\star(x)+g^\star(y) - \tfrac12\|x-y\|^2)/\eps)}\cN(z|ty + (1-t)x, t(1-t)\eps)\mu({\rm d}x) \nu({\rm d}y)\\
&= \int e^{g^\star(y)/\eps}\cN(z|y,(1-t)\eps) \nu({\rm d}y)  \int e^{f^\star(x)/\eps}\cN(z|x,t\eps)\mu({\rm d}x) \\
	&= \cH_{1-t}[\exp(g^\star/\eps)\nu](z)\cH_{t}[\exp(f^\star/\eps) \mu](z)
\end{align*}
where throughout we use $\cN(z|m,\sigma^2)$ to denote the Gaussian density with mean $m$ and covariance $\sigma^2 I$, and the fourth equality follows from computing the explicit density of the product of two Gaussians.

Also, \citet[][Proposition 4.1]{leonard2013survey} shows that when $\mu$ and $\nu$ have finite entropy, the optimal drift in \eqref{eq:dyneot2} is given by
\begin{align*}
\begin{split}
b_{t}^\star(z) &= \eps \nabla \log \cH_{(1-t)\eps}[\exp(g^\star/\eps)\nu](z)\,,
\end{split} 
\end{align*}
whence the pair $(\msf p_{t}^\star,b_{t}^\star)$ satisfies the Fokker--Planck equation.
This fact implies that if $X_t$ solves
\begin{align}\label{eq:sb_drift}
    \dd X_t = b_t^\star(X_t) \dd t + \sqrt{\eps}\dd B_t\,, \quad \quad X_0 \sim \mu\,,
\end{align}
then $\msf p^*_t = \mathrm{Law}(X_t)$.
In fact, more is true: the SDE~\eqref{eq:sb_drift} give rise to a path measure, which exactly agrees with the Schr\"odinger bridge.
Though \citet{leonard2013survey} derives these facts for $\mu$ and $\nu$ with finite entropy, we show in \cref{prop:selfconsistency}, below, that they hold in more generality.

Further developing the expression for $b_t^\star$, we obtain
\begin{align}\label{eq:bteps_main}
\begin{split}
    b_{t}^\star(z) 
    &= (1-t)^{-1}\Bigl(-z + \frac{\int y e^{(g^\star(y) - \tfrac{1}{2(1-t)}\|z-y\|^2)/\eps}\dd\nu(y)}{\int e^{(g^\star(y) - \tfrac{1}{2(1-t)}\|z-y\|^2)/\eps}\dd\nu(y)} \Bigr) \\
    &\eqqcolon (1-t)^{-1}(-z 
    + \nabla\phi_{1-t}^\star(z))\,. \\
\end{split}
\end{align}
Thus, our final expression for the SDE that yields the Schr\"odinger bridge is
\begin{align}\label{eq:sde_main}
    \dd X_t = (-(1-t)^{-1}X_t + (1-t)^{-1}\nabla\phi_{1-t}^\star(X_t)) \dd t + \sqrt{\eps}\dd B_t\,.
\end{align}
{Once again, we emphasize that our choice of notation here is deliberate: the drift is expressed as a function of a particular entropic Brenier map, namely, the entropic Brenier map between $\msf p_t^\star$ and $\nu$ with regularization parameter $(1-t)\eps$.}

We summarize this collection of crucial properties in the following proposition; see~\cref{sec:proof_selfconsistency} for proofs.
We note that this result avoids the finite entropy requirements of analogous results in the literature~\citep{leonard2013survey,shi2024diffusion}.
\begin{proposition}\label{prop:selfconsistency}
Let $\pi$ be a \emph{probability} measure of the form
\begin{align}\label{eq:pi_generic}
    \pi({\rm d}x_0, {\rm d}x_1) = \Lambda_\eps\exp((f(x_0) + g(x_1) - \tfrac12\|x_0-x_1\|^2)/\eps)\mu_0({\rm d}x_0) \mu_1({\rm d}x_1)\,,
\end{align}
for \emph{any} measurable $f$ and $g$ and \emph{any} probability measures $\mu_0,\mu_1 \in \cP_2(\R^d)$. Let $\msf M$ the path measure given by a mixture of Brownian bridges with respect to \eqref{eq:pi_generic} as in \eqref{eq:sb_first}, with  $t$-marginals $\msf m_t$ for $t \in [0,1]$. The following hold:
\begin{enumerate}
    \item The path measure $\msf M$ is Markov;
    \item The marginal $\msf m_t$ is given by
    \begin{equation*}
    	\msf m_{t}({\rm d}z) = \cH_{(1-t)\eps}[\exp(g/\eps)\mu_1](z)\cH_{t\eps}[\exp(f/\eps) \mu_0](z) {\rm d}z\,;
    \end{equation*}
    \item  $\msf M$ is the law of the solution to the SDE
\begin{align*}
    \dd X_t = \eps \nabla \log \cH_{(1-t)\eps}[\exp(g/\eps) \mu_1](X_t)\dd t + \sqrt{\eps}\dd B_t\,, \quad X_0 \sim \mu_0\,;
\end{align*}
    \item The drift above can be expressed as $b_t(z) = (1-t)^{-1}(z - \nabla\phi_{1-t}(z))$, where $\nabla \phi_{1-t}$ is the entropic Brenier map between $\msf m_t$ and $\rho$ with regularization strength $(1-t)\eps$, where 
    \begin{align*}
    \rho({\rm d}x_1) = \mu_1({\rm d}x_1)\exp\bigl(g(x_1)/\eps + \log \cH_\eps[e^{f/\eps}\mu_0](x_1)\bigr)\,.
\end{align*}
    If \eqref{eq:pi_generic} is the \emph{optimal} entropic coupling between $\mu_0$ and $\mu_1$, then $\rho \equiv \mu_1$.
\end{enumerate}
\end{proposition}

\subsection{Defining the estimator}\label{sec:defining_bridge}
In light of \eqref{eq:bteps_main}, it is easy to define an estimator on the basis of samples. Let $X_1,\ldots,X_m \sim \mu$ and $Y_1,\ldots, Y_n \sim \nu$, and let $\mu_m \defeq m^{-1}\sum_{i=1}^m \delta_{X_{i}}$, and similarly $\nu_n \defeq n^{-1}\sum_{j=1}^n \delta_{Y_j}$. Let $(\hat{f},\hat{g}) \in \R^m \times \R^n$ be the optimal entropic potentials associated with $\OTeps(\mu_m,\nu_n)$, which can be computed efficiently via Sinkhorn's algorithm \citep{cuturi2013sinkhorn,PeyCut19} with a  runtime of $O(mn/\eps)$ \citep{AltWeeRig17}. A natural plug-in estimator for the optimal drift is thus
\begin{align}\label{eq:hatbteps}
\begin{split}
\hat{b}_{t}(z) &\defeq \eps \nabla \log \cH_{(1-t)\eps}[\exp(\hat g/\eps) \nu_n] \\
    &= (1-t)^{-1}\Bigl(-z + \frac{\sum_{j=1}^n Y_j \exp\bigl((\hat{g}_j - \tfrac{1}{2(1-t)}\|z-Y_j\|^2)/\eps\bigr)}{\sum_{j=1}^n \exp\bigl((\hat{g}_j - \tfrac{1}{2(1-t)}\|z-Y_j\|^2)/\eps\bigr)} \Bigr) \\
    & =: (1-t)^{-1}(-z + \nabla\hat\phi_{1-t}(z)) \\
\end{split}
\end{align}
Further discussions on the numerical aspects of our estimator are deferred to \cref{sec:numerics}. Since we want to estimate the path given by $\msf{P}^\star$, our estimator is given by the solution to the following SDE:
\begin{align}\label{eq:disc_sde}
    \dd \hat{X}_t = (-(1-k\eta)^{-1}\hat{X}_{k\eta} + (1-k\eta)^{-1}\nabla\hat\phi_{1-k\eta}(\hat{X}_{k\eta}))\dd t + \sqrt{\eps}\dd B_t\,,
\end{align}
for $t \in [k\eta,(k+1)\eta]$, where $\eta \in (0,1)$ is some step-size, and $k$ is the iteration number. Though it is convenient to write the drift in terms of a time-varying entropic Brenier map, \eqref{eq:hatbteps} shows that for all $t \in (0, 1)$, our estimator is a simple function of the potential $\hat g$ obtained from a single call to Sinkhorn's algorithm.

\begin{remark}
	To the best of our knowledge, the idea of using static potentials to estimate the SB drift was first explored by \cite{finlay2020learning}. However, their proposal had some inconsistencies. 
For example, they assume a finite entropy condition on the source and target measures, and perform a standard Gaussian convolution on $\R^d$ instead of our proposed convolution $\cH_{(1-t)\eps}[\exp(\hat g/\eps) \nu_n]$.
	The former leads to a computationally intractable estimator, whereas, as we have shown above, the former has a simple form that is trivial to compute.
\end{remark}

\begin{remark}
An alternative approach to computing the Schr{\"o}dinger bridge is due to \cite{stromme2023sampling}: Given $n$ samples from the source and target measure, one can efficiently compute the in-sample entropic optimal coupling $\hat{\pi}$ on the basis of samples via Sinkhorn's algorithm. 
Resampling a pair $(X', Y') \sim \hat{\pi}$ and computing the Brownian bridge between $X'$ and $Y'$ yields an approximate sample from the Schr\"odinger bridge.
We remark that the computational complexity of our approach is significantly lower than that of \cite{stromme2023sampling}. While both methods use Sinkhorn's algorithm to compute an entropic optimal coupling between the source and target measures, Stromme's estimator necessitates $n$ \emph{fresh} samples from $\mu$ and $\nu$ to obtain a single approximate sample from the SB.
By contrast, having used our method to estimate the drifts, fresh samples from $\mu$ can be used to generate unlimited approximate samples from the SB.
\end{remark}

\section{Main results and proof sketch}\label{sec:proof_sketch}
We now present the proof sketches to our main result. We first present a sketch focusing purely on the statistical error incurred by our estimator, and later, using standard tools \citep{chen2022sampling,lee2023convergence}, we incorporate the additional time-discretization error. All omitted proofs in this section are deferred to \cref{app:proofs}.

\subsection{Statistical analysis}\label{sec:stat_analysis}
We restrict our analysis to the one-sample estimation task, as it is the closest to real-world applications where the source measure is typically known (e.g., the standard Gaussian) and the practitioner is given finitely many samples from a distribution of interest (e.g., images). Thus, we assume full access to $\mu$ and access to $\nu$ through i.i.d.~data, and let $(\hat f, \hat g)$ correspond to the optimal entropic potentials solving $\OTeps(\mu,\nu_n)$, which give rise to an optimal entropic plan  $\pi_n$.
Formally, this corresponds to the $m \to \infty$ limit of the setting described in \cref{sec:defining_bridge}; the estimator for the drift~\eqref{eq:hatbteps} is unchanged.

Let $\tilde{\msf P }$ be the Markov measure associated with the mixture of Brownian bridges defined with respect to $\pi_n$. By \cref{prop:selfconsistency}, the $t$-marginals are given by 
\begin{align}
    \tilde{\msf p}_t(z) = \cH_{(1-t)\eps}[\exp(\hat g/\eps)\nu_n](z)\cH_{t\eps}[\exp(\hat f/\eps)\mu](z)\,,
\end{align}
and the one-sample empirical drift is equal to
\begin{align*}
    \hat{b}_t(z) = \eps \nabla \log \cH_{(1-t)\eps}[\exp(\hat g/\eps)\nu_n](z)\,.
\end{align*}
Thus, $\tilde{\msf P}$ is the law of the following process with $\tilde{X}_0 \sim \mu$
\begin{align}\label{eq:hatbt_sde}
    \dd \tilde{X}_t = \hat{b}_t(\tilde{X}_t) \dd t + \sqrt{\eps}\dd B_t\,.
\end{align}
Note that this agrees with our estimator in~\eqref{eq:disc_sde}, but without discretization.
This process is not technically implementable, but forms an important theoretical tool in our analysis.

Our main result of this section is the following theorem.
\begin{theorem}[One-sample estimation; no discretization]\label{thm:onesamp}
    Suppose both $\mu,\nu \in \cP_2(\R^d)$, and $\nu$ is supported on a $\msf k$-dimensional smooth submanifold of $\R^d$ whose support is contained in a ball of radius $R > 0$. Let $\tilde{\msf P}$ (resp. $\msf P$) be the path measure corresponding to \eqref{eq:hatbt_sde} (resp. \eqref{eq:bteps_main}). Then it holds that, for any $\tau \in [0,1)$,
    \begin{align*}
        \E[\tvsq{\tilde{\msf P}_{[0,\tau]}}{\msf P_{[0,\tau]}^\star}]
        \lesssim \Bigl( \frac{\eps^{-\msf k/2 - 1}}{\sqrt{n}} + \frac{R^2\eps^{-\msf k}}{(1-\tau)^{\msf k + 2} n} \Bigr)\,.
    \end{align*}
\end{theorem}

As mentioned in the introduction, the parametric rates will not be surprising given the proof sketch below, which incorporates ideas from entropic optimal transport. The rates diverge exponentially in $\msf k$ as $\tau \to 1$; this is a consequence of the fact that the estimated drift $\hat b_t$ enforces that the samples exactly collapse onto the training data at terminal time, which is far from the true target measure.

The proof of \cref{thm:onesamp} uses key ideas from \cite{stromme2023minimum}: We introduce the following entropic plan 
\begin{align}\label{eq:pibar}
    \bar{\pi}_n(x,y) \defeq \Lambda_\eps\exp\bigl((\bar{f}(x) + g^\star(y) - \tfrac12\|x-y\|^2)/\eps\bigr)\mu({\rm d}x) \nu_n({\rm d}y)\,,
\end{align}
where $g^\star$ is the optimal entropic potential for the population measures ($\mu$, $\nu$), and
where we call $\bar{f} :\R^d \to \R$ a \emph{rounded} potential, defined as
\begin{align*}
    \bar{f}(x) \defeq -\eps \log \Bigl(\Lambda_\eps \cdot n^{-1}\sum_{j=1}^n \exp((g^\star(Y_j) - \tfrac12\|x-Y_j\|^2)/\eps)\Bigr)\,.
\end{align*}
Note that $\bar{f}$ can be viewed as the Sinkhorn update involving the potential $g^\star$ and measure $\nu_n$, and that $\bar{\pi}_n \in \Gamma(\mu,\bar{\nu}_n)$, where $\bar{\nu}_n$ is a rescaled version of $\nu_n$.
{We again exploit \cref{prop:selfconsistency}. Consider the path measure associated to the mixture of Brownian bridges with respect to $\bar{\pi}_n$, denoted $\bar{\msf P}$ (with $t$-marginals $\bar{\msf{p}}_t$), which corresponds to an SDE with drift}
\begin{align}\label{eq:barbt}
\begin{split}
    \bar{b}_t(z) &= \eps \nabla \log \cH_{1-t}[\exp(g^\star/\eps)\nu_n](z) \\
    &= (1-t)^{-1}\Bigl(-z + \frac{\sum_{j=1}^NY_j \exp((g^\star(Y_j) + \tfrac{1}{2(1-t)}\|z-Y_j\|^2)/\eps) }{\sum_{j=1}^N\exp((g^\star(Y_j) + \tfrac{1}{2(1-t)}\|z-Y_j\|^2)/\eps) }\Bigr)\,.    
\end{split}
\end{align}
Introducing the path measure $\bar{\msf{P}}_{[0,\tau]}$ into the bound via triangle inequality and then applying Pinsker's inequality, we arrive at
\begin{align*}
    \E[\tvsq{\tilde{\msf P}_{[0,\tau]}}{\msf P_{[0,\tau]}^\star}] &\lesssim \E[\tvsq{\tilde{\msf P}_{[0,\tau]}}{\bar{\msf P}_{[0,\tau]}}] + \E[\tvsq{\bar{\msf P}_{[0,\tau]}}{\msf P_{[0,\tau]}^\star}] \\
    &\lesssim \E[\kl{\tilde{\msf P}_{[0,\tau]}}{\bar{\msf P}_{[0,\tau]}}] + \E[\kl{{\msf P}_{[0,\tau]}^\star}{\bar{\msf P}_{[0,\tau]}}]\,,
\end{align*}
We analyse the two terms separately, each term involving proof techniques developed by \cite{stromme2023minimum}. We summarize the results in the following propositions, which yield the proof of \cref{thm:onesamp}.
\begin{proposition}\label{prop:kl_path_prop}
    Assume the conditions of \cref{thm:onesamp}, then for any $\tau \in [0,1)$
    \begin{align*}
        \E[\kl{\tilde{\msf P}_{[0,\tau]}}{\bar{\msf P}_{[0,\tau]}}] \leq \frac{1}{\eps}\E[\OTeps(\mu,\nu_n) - \OTeps(\mu,\nu)] \leq \eps^{-(\msf k/2 + 1)}n^{-1/2}\,.
    \end{align*}
\end{proposition}

\begin{proposition}\label{prop:kl_grisanov_prop}
    Assume the conditions of \cref{thm:onesamp}, then
    \begin{align*}
        \E[\kl{{\msf P}_{[0,\tau]}^\star}{\bar{\msf P}_{[0,\tau]}}] \leq \frac{R^2 \eps^{-\msf k}}{n}(1-\tau)^{-\msf k - 2}\,.
    \end{align*}
\end{proposition}

\subsection{Completing the results}\label{sec:completing}
{We now incorporate the discretization error. Letting $\hat{\msf P}$ denote the path measure induced by the dynamics of \eqref{eq:disc_sde}, we use the triangle inequality to introduce the path measure $\tilde{\msf P}$:
\begin{align*}
    \E[\tvsq{\hat{\msf P}_{[0,\tau]}}{\msf P_{[0,\tau]}^\star}] \lesssim \E[\tvsq{\hat{\msf P}_{[0,\tau]}}{\tilde{\msf P}_{[0,\tau]}}] + \E[\tvsq{\tilde{\msf P}_{[0,\tau]}}{{\msf P}_{[0,\tau]}^\star}]\,.
\end{align*}
The second term is precisely the statistical error, controlled by \cref{thm:onesamp}.
For the first term, we employ a now-standard discretization argument (see e.g., \cite{chen2022sampling}) which bounds the total variation error as a function of the step-size parameter $\eta$ and the Lipschitz constant of the empirical drift, which can be easily bounded in our setting.}

\begin{proposition}\label{prop:disc_error}
Suppose $\mu,\nu \in \cP_2(\R^d)$. Denoting $L_\tau$ for the Lipschitz constant of $\hat{b}_\tau$ (recall \cref{eq:hatbteps}) for $t \in [0,1)$ and $\eta$ the step-size of the SDE discretization, it holds that
 \begin{align*}
    \E[\tvsq{\hat{\msf P}_{[0,\tau]}}{\tilde{\msf P}_{[0,\tau]}}] \lesssim (\eps+1)L_\tau^2 d \eta\,.
\end{align*}
In particular, if $\supp(\nu)\sse B(0;R)$, then
\begin{align*}
    \E[\tvsq{\hat{\msf P}_{[0,\tau]}}{\tilde{\msf P}_{[0,\tau]}}]  \lesssim (\eps+1)(1-\tau)^{-2} d \eta (1 \vee R^4(1-\tau)^{-2}\eps^{-2})\,.
\end{align*}
\end{proposition}

We now aggregate the statistical and approximation error into one final result.
\begin{theorem}\label{thm:agg_thm}
    Suppose $\mu, \nu \in \cP_2(\R^d)$ with $\supp(\nu) \sse B(0,R) \sse \cM$, where $\cM$ is a $\msf k$-dimensional submanifold of $\R^d$. Given $n$ i.i.d.~samples from $\nu$, the one-sample Sinkhorn bridge $\hat{\msf P}$ estimates the Schr{\"o}dinger bridge ${\msf P}^\star$ with the following error
    \begin{align*}
        \E[\tvsq{\hat{\msf P}_{[0,\tau]}}{{\msf P}^\star_{[0,\tau]}}]  &\lesssim \Bigl( \frac{\eps^{-\msf k/2 - 1}}{\sqrt{n}} + \frac{R^2\eps^{-\msf k}}{(1-\tau)^{\msf k + 2} n} \Bigr) \\
        &\qquad + (\eps+1)(1-\tau)^{-2} d \eta (1 \vee R^4(1-\tau)^{-2}\eps^{-2})\,.
    \end{align*}
Assuming $R \geq 1$ and $\eps = 1$, the Schr\"odinger bridge can be estimated in total variation distance to accuracy $\epsilon_{\mathrm{TV}}$ with $n$ samples and $N$ Euler--Maruyama steps, where 
\begin{align*}
    n \asymp \frac{R^2}{(1-\tau)^{\msf k + 2}\epsilon_{\mathrm{TV}}^2} \vee \epsilon_{\mathrm{TV}}^{-4}\,, \quad N \lesssim \frac{dR^4}{(1-\tau)^4\epsilon_{\mathrm{TV}}^2}\,.
\end{align*}
\end{theorem}
Note that our error rates improve as $\eps \to \infty$; since this is also the regime in which Sinkhorn's algorithm terminates rapidly, it is natural to suppose that $\eps$ should be large in practice.
This is misleading, however: as $\eps$ grows, the Schr\"odinger bridge becomes less and less informative,\footnote{In other words, the transport path is more and more volatile.} and the marginal $\msf p^\star_\tau$ only resembles $\nu$ when $\tau$ becomes very close to $1$. We elaborate on the use of the SB for sampling in the following section.

{\subsection{Application: Sampling with the F{\"o}llmer bridge}\label{sec:follmer_sampling}
\cref{thm:agg_thm} does not immediately imply guarantees for sampling from the target distribution $\nu$.
Obtaining such guarantees requires arguing that simulating the Sinkhorn bridge on a suitable interval $[0, \tau]$ for $\tau$ close to $1$ yields samples close to the true density (without completely collapsing onto the training data).
We provide such a guarantee in this section, for the special case of the F\"ollmer bridge.
We adopt this setting only for concreteness; similar arguments apply more broadly.

The F\"ollmer bridge is a special case of the Schr{\"o}dinger bridge due to Hans F{\"o}llmer~\citep{Fol85}. In this setting, $\mu=\delta_a$ for any $a \in \R^d$,
and our estimator takes a particularly simple form:\begin{align}\label{eq:estdrift_follmer}
    \hat{b}^{\msf F}_{t}(z) = (1-t)^{-1}\Bigl(-z + \frac{\sum_{j=1}^n Y_j \exp\bigl(( \tfrac12\|Y_j\|^2  - \tfrac{1}{2(1-t)}\|z-Y_j\|^2)/\eps\bigr)}{\sum_{j=1}^n \exp\bigl((\tfrac12\|Y_j\|^2 - \tfrac{1}{2(1-t)}\|z-Y_j\|^2)/\eps\bigr)} \Bigr)\,,
\end{align}
Note that in this special case, calculating the drift does not require the use of Sinkhorn's algorithm, and the drift, in fact, corresponds to the score of a kernel density estimator applied to $\nu_n$.
We provide a calculation of these facts in \cref{sec:follmer_calcs} for completeness.} 

We then have the following guarantee.

\begin{corollary}\label{cor:sampling_follmer_prop}
Consider the assumptions of \cref{thm:agg_thm}, further suppose that $\mu = \delta_0$ and $\eps = 1$ and that the second moment of $\nu$ is bounded by $d$. Suppose we use $n$ samples from $\nu$ to estimate the F{\"o}llmer drift, and simulate the resulting SDE using $N$ Euler--Maruyama iterations until time $\tau = 1 - \epsilon^2_{\mathrm{W}_2}/d$, with
\begin{align*}
    n \asymp \frac{R^2 d^{\msf k + 2}}{\epsilon_{\mathrm{W}_2}^{\msf 2k + 4}\epsilon_{\mathrm{TV}}^2} \vee \epsilon_{\mathrm{TV}}^{-4} \qquad N \lesssim \frac{R^4 d^5}{\epsilon_{\mathrm{W}_2}^{8}\epsilon_{\mathrm{TV}}^2}\,.
\end{align*}
Then the density given by the Sinkhorn bridge at time $\tau$ iterations will be $\epsilon_{\mathrm{TV}}$-close in total variation to a measure which is $\epsilon_{\mathrm{W}_2}$-close to $\nu$ in the $2$-Wasserstein distance.\looseness-1
\end{corollary}
Note that the choice $\eps = 1$ was merely out of convenience. If instead the practitioner was willing to pay the computational price of solving Sinkhorn's algorithm for small $\eps$ and large $n$, then the number of requisite iterations $N$ would decrease. Finally, notice that the number of samples scales exponentially in the intrinsic dimension $\msf k \ll d$ instead of the ambient dimension $d$. This is, of course, unavoidable, but improves upon recent work that uses kernel density estimators to prove a similar result for denoising diffusion probabilistic models \citep{wibisono2024optimal}.

\begin{remark}
Recently, \citet{huang2024one} also proposed \eqref{eq:estdrift_follmer} to estimate the F{\"o}llmer drift. They provide no statistical estimation guarantees of the drift, nor any sampling guarantees; their contributions are largely empirical, demonstrating that the proposed estimator is tractable for high-dimensional tasks. The work of \cite{huang2021schr} also proposes an estimator for the F{\"o}llmer bridge based on having partial access to the log-density ratio of the target distribution (without the normalizing constant). 
\end{remark}

\section{Numerical performance}\label{sec:numerics}
Our approach is summarized in \cref{alg:sinkhornbridges}, and open-source code for replicating our experiments is available at \href{https://github.com/APooladian/SinkhornBridge}{https://github.com/APooladian/SinkhornBridge.}\footnote{Our estimator is implemented in both the \href{https://pythonot.github.io/}{POT} and \href{https://ott-jax.readthedocs.io/en/latest/}{OTT-JAX} frameworks.} 

For a fixed regularization parameter $\eps > 0$, the runtime of computing $(\hat{f},\hat{g})$ on the basis of samples has complexity $\cO(mn/(\eps\delta_{\text{tol}}))$, where $\delta_{\text{tol}}$ is a required tolerance parameter that measures how closely the the marginal constraints are satisfied \citep{cuturi2013sinkhorn,PeyCut19,altschuler2021asymptotics}. Once these are computed, the evaluation of $\hat{b}_{k\eta}$ is $\cO(n)$, with the remaining runtime being the number of iteration steps, denoted by $N$. In all our experiments, we take $m=n$, thus the total runtime complexity of the algorithm is a fixed cost of $\cO(n^2/(\eps\delta_{\text{tol}})$, followed by $\cO(nN)$ for each new sample to be generated (which can be parallelized). 

\begin{algorithm}[t]
\caption{Sinkhorn bridges}
\begin{algorithmic}
\State \textbf{Input:} Data $\{X_{i}\}_{i=1}^m \sim \mu$, $\{Y_{j}\}_{j=1}^n \sim \nu$, parameters $\eps > 0$, $\tau \in (0,1)$, and $N \geq 1$
\State \textbf{Compute:} Sinkhorn potentials $(\hat{f},\hat{g}) \in \R^m \times \R^n$ \Comment{Using \href{https://pythonot.github.io/}{POT} or \href{https://ott-jax.readthedocs.io/en/latest/}{OTT}}
\State \textbf{Initialize:} $x^{(0)} = x \sim \mu$, $k = 0$, stepsize $\eta = 
\tau/N$
\While{$k \le N-1$}
\State $x^{(k+1)} = x^{(k)} + \eta \hat{b}_{k\eta}(x^{(k)}) + \sqrt{\eta\eps}\xi$ \Comment{$\xi \sim \cN(0,I)$} 
\State $k \gets k+1$
\EndWhile
\State \textbf{Return:} $x^{(N)}$ 
\end{algorithmic}\label{alg:sinkhornbridges}
\end{algorithm}

\subsection{Qualitative illustration}
As a first illustration, we consider standard two-dimensional datasets from the machine learning literature. For all examples, we use $n=2000$ training points from both the source and target measure, and run Sinkhorn's algorithm with $\eps = 0.1$. For generation, we set $\tau=0.9$, and consider $N=50$ Euler--Maruyama steps. \Cref{fig:sb_toydata} contains the resulting simulations, starting from out-of-sample points. We see reasonable performance in each case.

\begin{figure}[t]
\centering
    \begin{subfigure}{\textwidth}            
            \includegraphics[width=\textwidth]{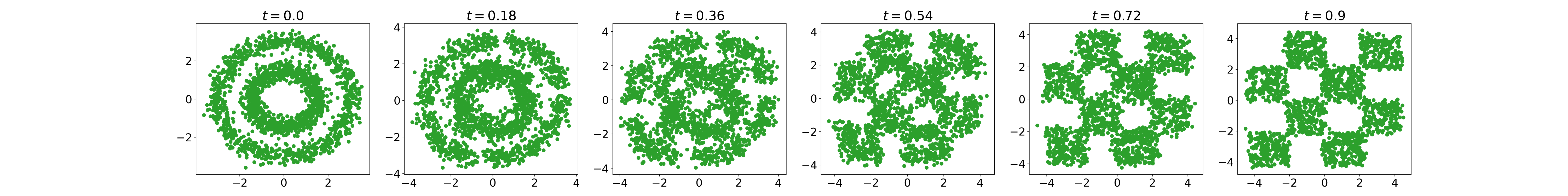}
    \end{subfigure}\\
     \begin{subfigure}{\textwidth}
            \centering
            \includegraphics[width=\textwidth]{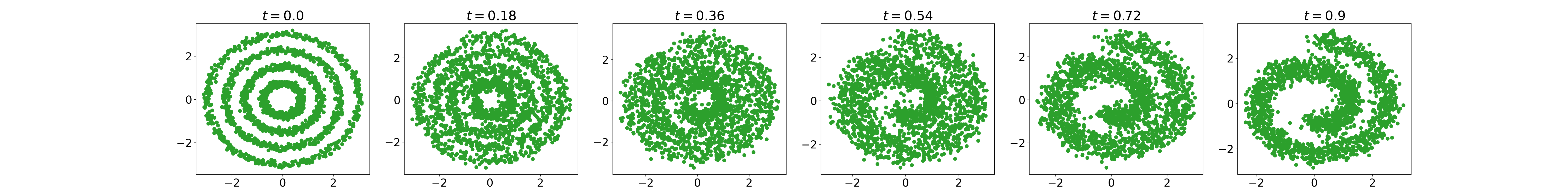}
    \end{subfigure}\\
     \begin{subfigure}{\textwidth}
            \centering
            \includegraphics[width=\textwidth]{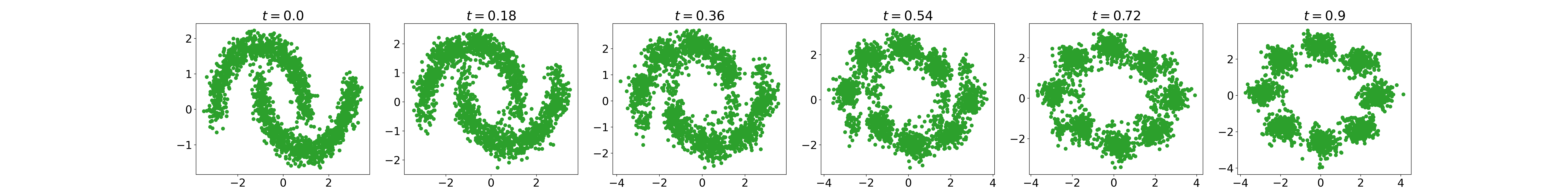}
    \end{subfigure}
    \caption{Schr{\"o}dinger bridges on the basis of samples from toy datasets. }
    \label{fig:sb_toydata}
\end{figure}

\subsection{Quantitative illustrations}\label{sec:}
We quantitatively assess the performance of our estimator using synthetic examples from the deep learning literature \citep{bunne2023schrodinger,gushchin2023building}.

\subsubsection{The Gaussian case}\label{sec:gaussian_bridge}
We first demonstrate that we are indeed learning the drift and that the claimed rates are empirically justified. As a first step, we consider the simple case where $\mu = \cN(a,A)$ and $\nu = \cN(b,B)$ for two positive-definite $d \times d$ matrices $A$ and $B$ and arbitrary vectors $a,b\in\R^d$. In this regime, the optimal drift $b_\tau^\star$ and $\msf p_\tau^\star$ has been computed in closed-form by \citet{bunne2023schrodinger}; see equations (25)-(29) in their work.

To verify that we are indeed learning the drift, we first draw $n$ samples from $\mu$ and $\nu$, and compute our estimator, $\hat{b}_\tau$ for any $\tau \in [0,1)$. We then evaluate the mean-squared error
\begin{align*}
    \mathrm{MSE}(n,\tau) = \|\hat{b}_\tau - b_\tau^\star\|^2_{L^2(\msf p_\tau^\star)}\,,
\end{align*}
by a Monte Carlo approximation, with $n_{\mathrm{MC}} = 10000$. For simplicity, with $d=3$, we choose $A = I$ and randomly generate a positive-definite matrix $B$, and center the Gaussians. We fix $\eps=1$ and vary $n$ used to define our estimator, and perform the simulation ten times to generate error bars across various choices of $\tau \in [0,1)$; see \cref{fig:gaussian_drift}.

\begin{figure}[t]
    \centering
\includegraphics[width=0.65\linewidth]{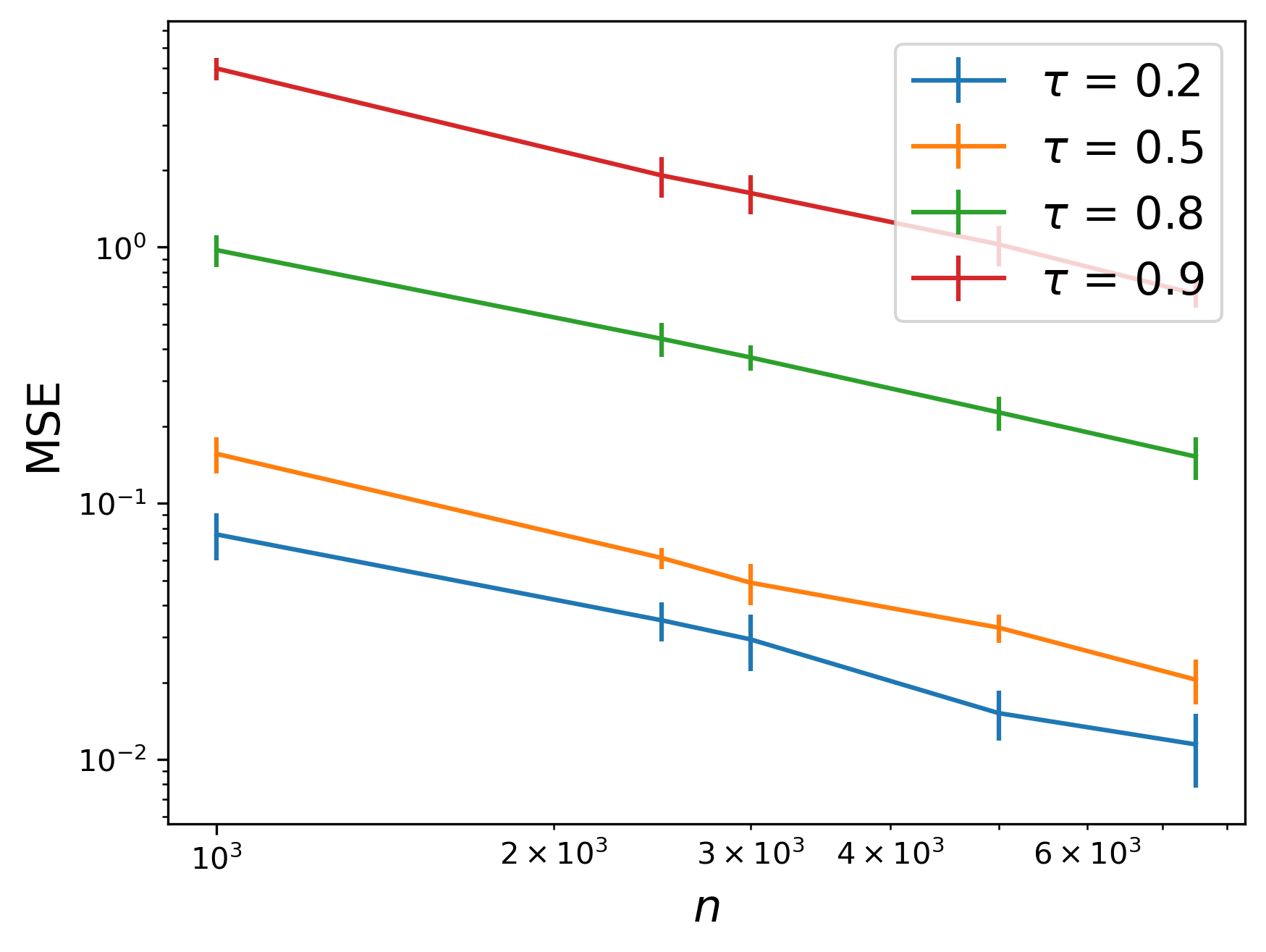}
    \caption{MSE for estimating the Gaussian drift as $(n,\tau)$ vary, averaged over 10 trials.}
    \label{fig:gaussian_drift}
\end{figure}

It is clear from the plot that the \emph{constant} associated to the rate of estimation gets worse as $\tau \to 1$, but the overall rate of convergence appears unchanged, which hovers around $n^{-1}$ for all choices of $\tau$ shown in the plot, as expected from e.g., \cref{prop:kl_path_prop}.

\subsubsection{Multimodal measures with closed-form drift}\label{sec:gmm}
The next setting is due to \citet{gushchin2023building}; they devised a drift that defines the Schr{\"o}dinger bridge between a Gaussian and a more complicated measure with multiple modes. This explicit drift allowed them to benchmark multiple neural network based methods for estimating the Schr{\"o}dinger bridge for non-trivial couplings (e.g., beyond the Gaussian to Gaussian setting). We briefly remark that the approaches discussed in their work fall under the ``continuous estimation" paradigm, where researchers assume they can endlessly sample from the distributions when training (using new samples per training iteration).

We consider the same pre-fixed drift as found in their publicly available code, which transports the standard Gaussian to a distribution with four modes. We consider the case $d=64$ and $\eps = 1$, as these hyperparameters are most extensively studied in their work, where they provide the most details on the other models. We use $n = 4096$ training samples from the source and target data they construct (which is significantly less than the total number of samples required for any of the neural network based models) and perform our estimation procedure, and we take $N=100$ discretization steps (which is half as many as most of the works they consider) to simulate to time $\tau=0.99$. To best illustrate the four mixture components, \cref{fig:gmm_1and15} contains a scatter plot of the first and fifteenth dimension, containing fresh target samples and our generated samples. 

\begin{figure}[t]
    \centering
    \begin{minipage}{0.45\textwidth}
        \centering
        \includegraphics[width=\linewidth]{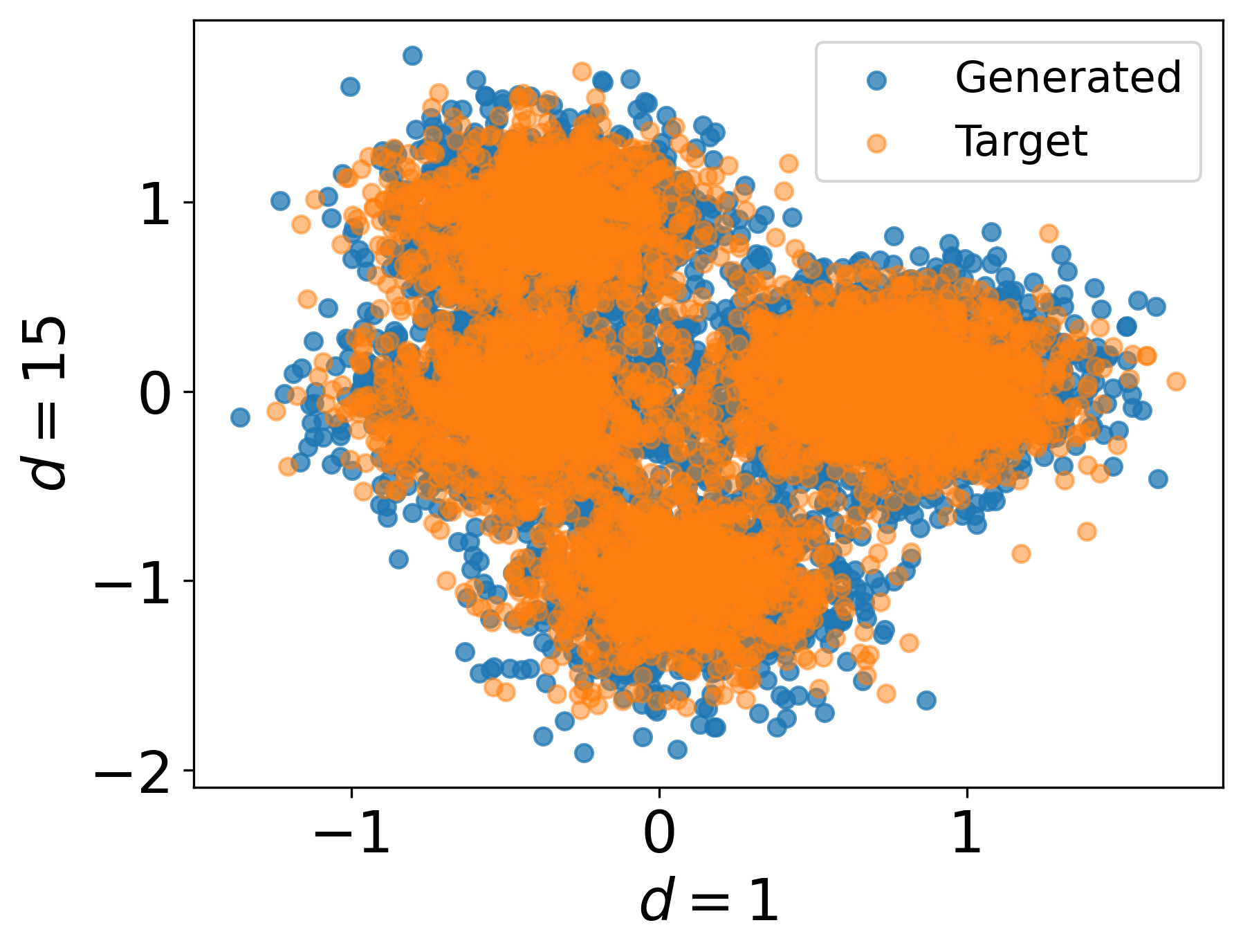} 
        \caption{Plotting generated and resampled target data in $d=64$.}
        \label{fig:gmm_1and15}
    \end{minipage}\hfill
    \begin{minipage}{0.45\textwidth}
        \centering
        \begin{tabular}{c|c}
    \begin{tabular}{c|ccccc}
       Method  & BW-UVP \\
       \hline
       Ours & \textbf{0.41 $\pm$ 0.03} \\
       MLE-SB & 0.56  \\
       EgNOT & 0.85 \\ 
       FB-SDE-A& 0.65 
    \end{tabular}
        \end{tabular}
        \captionof{table}{Comparison to neural network approaches in BW-UVP for $d=64$.}
        \label{tab:bw_uvp}
    \end{minipage}
\end{figure}

We compare to the ground-truth samples using the unexplained variance percentage (UVP) based on the Bures--Wasserstein distance \citep{bures1969extension}: 
\begin{align*}
    \mu \mapsto \text{BW-UVP}_\nu(\mu) \defeq 100\frac{\text{BW}^2(\cN_\mu,\cN_\nu)}{0.5\cdot\text{Var}(\nu)}\,,
\end{align*}
where $\cN_\mu = \cN(\E_\mu[X],\text{Cov}_\mu(X))$, and same for $\cN_\nu$.\footnote{For us, these quantities are computed on the basis of samples.} While seemingly ad hoc, the BW-UVP is widely used in the machine learning literature as a means of quantifying the quality of the generated samples (see e.g., \citet{daniels2021score}). We compute the BW-UVP with $10^4$ generated samples from the target and our approach, averaged over 5 trials, and used the results of \citet{gushchin2023building} for the remaining methods (MLE-SB is by \citet{vargas2021solving}, EgNOT is by \citet{mokrov2023energy}, and FB-SDE-A is by \citet{chen2021likelihood}). We see that the Sinkhorn bridge has significantly lower BW-UVP compared to the other approaches while requiring less compute resources and training data.

\section{Conclusion}\label{sec:conclusion}
This work makes a connection between the static entropic optimal transport problem, the Schr{\"o}dinger bridge problem, and Sinkhorn's algorithm, which appeared to be lacking in the literature. We proposed and analyzed a plug-in estimator of the Schr{\"o}dinger bridge, which we call the Sinkhorn bridge. Due to a Markov property enjoyed by entropic optimal couplings, our estimator relates Sinkhorn's matrix-scaling algorithm to the optimal drift that arises in the Schr{\"o}dinger bridge problem, and existing theory in the statistical optimal transport literature provide us with statistical guarantees. A novelty of our approach is the reduction of a ``dynamic" estimation problem to a ``static" one, where the latter is easy to analyze. 

Several questions arise from our work, we highlight some here:

\begin{description}
    \item \textbf{Further connections to other processes:} Our arguments for the Schr\"odinger bridge used the particular form of the reversible Brownian motion. It would be interesting to develop this approach for other types of reference processes for the purposes of developing statistical guarantees.     	
The Sinkhorn bridge estimator can also be implemented through an ordinary differential equation (ODE) and not necessarily through an SDE. This gives rise to the \emph{probability flow} ODE in the generative modeling literature \citep{song2020score}. \cite{chen2024probability} showed that this approach can achieve results comparable to those obtained by diffusion models \citep{chen2022sampling,lee2023convergence}.
    	We anticipate analogous results would hold in our setting. 
    \item \textbf{Lower bounds:} Entropic optimal transport suffers from a dearth of lower bounds in the literature. It is unclear whether our approach is optimal in terms of its dependence on $\eps$ and $\tau$. Developing estimators with better performance or nontrivial lower bounds would help establish how far our estimators are from optimality.
    \item \textbf{Computation in practice:} On the computational side, one can ask if are there better estimators of the drift $b_{t}^\star$ than the plug-in estimator we outlined (possibly amenable to statistical analysis), and to consider using our estimator on non-synthetic problems. For example, it seems advisable to compute the Sinkhorn bridge in a latent space, and reverting the latent transformation later \citep{rombach2022high}.
    \end{description}

\section*{Acknowledgements}
AAP thanks NSF grant DMS-1922658 and Meta AI Research for financial support. JNW is supported by the Sloan Research Fellowship and NSF grant DMS-2339829.

\bibliography{main}

\appendix
\section{Dynamic entropic optimal transport}\label{app:twoformulations}
\subsection{Connecting the two formulations}
In this section, we reconcile (at a formal level) two versions of the dynamic formulation for entropic optimal transport. We will start with \eqref{eq:dyneot2} and show that this is equivalent to \eqref{eq:dyneot} by a reparameterization.

We begin by recognizing that $\Delta\msf p_t = \nabla \cdot (\msf p_t \nabla \log \msf p_t)$, which allows us to write the Fokker--Planck equation as 
\begin{align}\label{eq:new_cont}
\partial_t \msf p_t + \nabla \cdot ((v_t - \tfrac{\eps}{2}\nabla \log \msf p_t)\msf p_t) = 0\,,
\end{align}
Inserting $b_t \defeq v_t - \tfrac{\eps}{2}\nabla \log \msf p_t$ into \eqref{eq:dyneot2}, we expand the square and arrive at
\begin{align*}
\inf_{(\msf p_t,b_t)} \int_0^1 \!\! \int \bigl(\frac12\|b_t(x)\|^2 + \frac{\eps^2}{8}\|\nabla \log \msf p_t(x)\|^2 + \frac{\eps}{2} b_t^\top \nabla \log \msf p_t\bigr) \msf p_t(x) \dd x \dd t\,.
\end{align*}
Up to the cross-term, this aligns with \eqref{eq:dyneot}; it remains to eliminate the cross term. Using integration-by-parts and \eqref{eq:new_cont}, we obtain
\begin{align*}
\int_0^1\!\!\int (b_t\msf p_t)^\top \nabla \log \msf p_t \dd x \dd t &= -\int_0^1\!\!\int \nabla \cdot (b_t\msf p_t) \log \msf p_t \dd x \dd t = \int_0^1\!\!\int  (\partial_t\msf p_t) \log\msf p_t \dd x \dd t\,.
\end{align*}
Though, we have (by product rule) the equivalence
\begin{align*}
\partial_t(\msf p_t \log \msf p_t) - \partial_t \msf p_t = (\partial_t \msf p_t) \log\msf p_t\,.
\end{align*}
Exchanging partial derivatives under the integral, this yields the following simplification
\begin{align*}
\int_0^1\!\!\int  (\partial_t\msf p_t) \log\msf p_t \dd x \dd t &= \int_0^1\!\!\int  \partial_t(\msf p_t\log\msf p_t) \dd x \dd t - \int_0^1\!\!\int  \partial_t\msf p_t \dd x \dd t \\
&= \int_0^1\!\! \partial_t \int  \msf p_t\log\msf p_t \dd x \dd t - \int_0^1\!\!  \partial_t \int\msf p_t \dd x \dd t  \\
&= \int_0^1\!\! \partial_t \cH(\msf p_t) \dd t + 0\\ 
&= \cH(\msf p_1) - \cH(\msf p_0)\,,
\end{align*}
where $\msf p_1 = \nu$ and $\msf p_0=\mu$. We see that \eqref{eq:dyneot2} is equivalent to
\begin{align*}
\frac{\eps}{2}(\cH(\nu) - \cH(\mu)) + \inf_{(\msf p_t,b_t)} \int_0^1 \!\! \int \Bigl(\frac12\|b_t(x)\|^2 + \frac{\eps^2}{8}\|\nabla \log \msf p_t(x)\|^2 \Bigr)\msf p_t(x)\dd x \dd t\,.
\end{align*}

\subsection{Connecting Markov processes and entropic Brenier maps}\label{sec:proof_selfconsistency}
Here we prove \cref{prop:selfconsistency}. To continue, we require the following lemma.
\begin{lemma}\label{lem:cond_indep}
	Fix any $t \in [0, 1]$.
	Under $\msf M$, the random variables $X_0$ and $X_1$ are conditionally independent given $X_t$.
\end{lemma}
\begin{proof}
	A calculation shows that the joint density of $X_0$, $X_1$, and $X_t$ with respect to $\mu_0 \otimes \mu_1 \otimes \mathrm{Lebesgue}$ equals
	\begin{multline*}
		\Lambda_\eps \Lambda_{t(1-t)\eps}
		e^{- \tfrac{1}{2\eps t(1-t)} \|x_t - ((1-t)x_0 + t x_1)\|^2} e^{(f(x_0) + g(x_1) - \tfrac12\|x_0-x_1\|^2)/\eps} \\= \msf F_t(x_t,x_0) \msf G_t(x_t,x_1)\,,
		\end{multline*}
		where
		\begin{align*}
		\msf F_t(x_t,x_0)& = 
		\Lambda_{\eps t}
		e^{f(x_0)/\eps}e^{-\tfrac{1}{2\eps t}\|x_t - x_0\|^2}\\ \msf G_t(x_t,x_1) &= 
		\Lambda_{(1-t)\eps}		
		e^{g(x_1)/\eps} e^{-\tfrac{1}{2\eps(1-t)}\|x_t - x_1\|^2}\,.
		\end{align*}
		Since this density factors, the law of $X_0$ and $X_1$ given $X_t$ is a product measure, proving the claim.
\end{proof}

\begin{proof}[Proof of \cref{prop:selfconsistency}]
First, we prove that $\msf M$ is Markov. Let $(X_t)_{t \in [0, 1]}$ be distributed according to $\msf M$.
It suffices to show that for any integrable $a \in \sigma(X_{[0,t]}), b \in \sigma(X_{[t,1]})$, we have the identity
\begin{equation*}
	\E[ab|X_t] = \E[a | X_t] \E[b | X_t] \quad \text{a.s.}
\end{equation*}
Using the tower property and the fact that, conditioned on $X_0$ and $X_1$, the law of the path is a Brownian bridge between $X_0$ and $X_1$, and hence is Markov, we have
\begin{align*}
	\E_{\msf M}[ab |X_t] & = \E[\E[ab |X_0, X_t, X_1] |X_t] \\
	& = \E[\E[a|X_0, X_t]\E[b|X_t, X_1] | X_t]\,.
\end{align*}
By~\cref{lem:cond_indep}, the sigma-algebras $\sigma(X_0,X_t)$ and $\sigma(X_t, X_1)$ are conditionally independent given $X_t$, hence
\begin{equation*}
	\E[\E[a|X_0, X_t]\E[b|X_t, X_1] | X_t] = \E[\E[a|X_0, X_t]|X_t] \E[\E[b|X_0, X_t]|X_t] = \E[a | X_t] \E[b | X_t]\,,
\end{equation*}
as claimed.

The proof of the second statement follows directly from the computations presented below~\eqref{eq:marginal-density}, which hold under no additional assumptions.

We now prove the third statement. Following the approach of~\cite{Fol85}, the representation of $\msf M$ as a mixture of Brownian bridges shows that the law of $X_{[0, t]}$ for any $t < 1$ has finite entropy with respect to the law of $X_0 + \sqrt{\eps} B_t$, for $X \sim \mu_0$.
Hence, to verify the representation in terms of the SDE, it suffices to compute the stochastic derivative:
\begin{equation*}
	\lim_{h \to 0} \frac 1h \EE[X_{t + h} - X_t | X_{[0,t]}]\,,
\end{equation*}
where the limit is taken in $L^2$.
Using the the fact that the process is Markov and, conditioned on $X_0$ and $X_1$, the path is a Brownian bridge, we obtain
\begin{align*}
	\lim_{h \to 0} \frac 1h \EE[X_{t + h} - X_t | X_{[0,t]}] & = \lim_{h \to 0} \frac 1h  \EE[\EE[X_{t+h}- X_t|X_0, X_t, X_1]|X_t]  \\
	& = \frac1{1-t}\EE[X_1 - X_t |X_t]\,.
\end{align*}
Recalling the computations in \cref{lem:cond_indep}, we observe that, conditioned on $X_t = x_t$, the variable $X_1$ has $\mu_1$ density proportional to $\msf G_t(x_t,x_1)$.
Since $\pi$ is a probability measure, in particular we have that $e^{g}$ lies in $L^1(\mu_1)$.
We can therefore apply dominated convergence to obtain
\begin{align*}
	\frac1{1-t}\EE[X_1 - X_t |X_t=x_t] = \frac{\int \tfrac{x_1 - x_t}{1-t} \msf G_t(x_t, x_1) \mu_1({\rm d} x_1)}{\int  \msf G_t(x_t, x_1) \mu_1({\rm d} x_1) } = \eps \nabla \log \cH_{(1-t)\eps}[\exp(g/\eps)\mu_1](x_t)\,,
\end{align*}
as desired.

For the fourth statement, we require the following claim.\\
\textbf{Claim:} The joint probability measure ${\pi}_{t}(z,x_1)$, defined as
\begin{align*}
\exp((-(1-t)f_{1-t}(z) + (1-t)g(x_1) - \tfrac12\|z - x_1\|^2))/((1-t)\eps))\msf{m}_{t}({\rm d}z)\mu_1({\rm d}x_1)\,,
\end{align*}
is the optimal entropic coupling from $\msf m_t$ to $\rho$ with regularization parameter $(1-t)\eps$, where $f_{1-t}(z) \defeq \eps \log \cH_{(1-t)\eps}[e^{g/\eps}\mu_1](z)$.
Under this claim, it is easy to verify that the definition of $\nabla \phi_{1-t}$ is precisely this conditional expectation, which concludes the proof. 

To prove the claim, we notice that $\pi_t$ is already in the correct form of an optimal entropic coupling, and ${\pi}_{t} \in \Gamma(\msf{m}_{t},?)$ by construction. Thus, it suffices to only check the second marginal. By the second part, above, we have that
\begin{align*}
    \msf m_t(z) = \cH_{(1-t)\eps}[\exp(g/\eps)\mu_1](z)\cH_{t\eps}[\exp(f/\eps)\mu_0](z)\,.
\end{align*}
Integrating, performing the appropriate cancellations, and applying the semigroup property, we have 
\begin{align*}
    \int \pi_{t}(z,{\rm d}x_1) \dd z &= e^{g(x_1)/\eps}\mu_1({\rm d}x_1)\cH_{(1-t)\eps}[\cH_{t\eps}[e^{f/\eps}\mu_0]](x_1) \\
    &= e^{g(x_1)/\eps}\mu_1({\rm d}x_1)\cH_{\eps}[e^{f/\eps}\mu_0](x_1)\,,
\end{align*}
which proves the claim. 
\end{proof}

\section{Proofs for Section \ref{sec:proof_sketch}}\label{app:proofs}
\subsection{One-sample analysis}
\begin{proof}[Proof of \cref{prop:kl_path_prop}]
First, we recognize that a path with law $\tilde{\msf P}$ (resp. $\bar{\msf P}$) can be obtained by sampling a Brownian bridge between $(X_0, X_1) \sim \pi_n$ (resp. $\bar{\pi}_n$), by \cref{prop:selfconsistency}. Thus, by the data processing inequality,
\begin{align*}
    \E[\kl{\tilde{\msf P}_{[0,\tau]}}{\bar{\msf P}_{[0,\tau]}}] & \leq \E[\kl{\tilde{\msf P}}{\bar{\msf P}}] \\
    & \leq \E[\kl{\pi_n}{\bar{\pi}_n}] \\
    &= \E\Bigl[\int \log(\pi_n/\bar{\pi}_n)\dd \pi_n\Bigr]\,,
\end{align*}
where the above manipulations are valid as both $\pi_n$ and $\bar{\pi}_n$ have densities with respect to $\mu \otimes \nu_n$. Completing the expansion by explicitly writing out the densities, we obtain
\begin{align*}
    \E[\kl{\tilde{\msf P}_{[0,\tau]}}{\bar{\msf P}_{[0,\tau]}}] &\leq \frac{1}{\eps}\E\Bigl[\int (\hat f + \hat g - \bar{f} - g^\star) \dd \pi_n\Bigr] \\
    &= \frac{1}{\eps}\E[\OTeps(\mu,\nu_n) - \int \bar{f}\dd \mu - \int g^\star \dd \nu_n]\,.
\end{align*}
We now employ the rounding trick of \citet{stromme2023minimum}: the rounded potential $\bar f$ satisfies
\begin{equation*}
	\bar f = \argmax_{f \in L^1(\mu)} \Phi^{\mu \nu_n}(f, g^\star)\,;
\end{equation*}
Therefore, in particular, $ \Phi^{\mu \nu_n}(\bar f, g^\star) \geq  \Phi^{\mu \nu_n}(f^\star, g^\star)$.
Continuing from above, we obtain
\begin{align*}
    \E[\kl{\tilde{\msf P}_{[0,\tau]}}{\bar{\msf P}_{[0,\tau]}}] &\leq \frac{1}{\eps}\E[\OTeps(\mu,\nu_n) - \int f^\star\dd \mu - \int g^\star \dd \nu_n] \\
    &= \frac{1}{\eps}\E[\OTeps(\mu,\nu_n) - \int f^\star\dd \mu - \int g^\star \dd \nu] \\
    &= \frac{1}{\eps}\E[\OTeps(\mu,\nu_n) - \OTeps(\mu,\nu)]\,,
\end{align*}
where in the penultimate equality we observed that $g$ is independent of the data $Y_1,\ldots,Y_n$. Combined with Theorem 2.6 of \citet{groppe2023lower}, the proof is complete.
\end{proof}

\begin{proof}[Proof of \cref{prop:kl_grisanov_prop}]
    We start by applying Girsanov's theorem to obtain a difference in the drifts, which can be re-written as differences in entropic Brenier maps:
    \begin{align}\label{eq:final_integrand}
    \begin{split}
        \E[\kl{{\msf P}_{[0,\tau]}^\star}{\bar{\msf P}_{[0,\tau]}}] &\leq \int_0^\tau \E \|\bar{b}_t - b_t^\star\|^2_{L^2(\msf p_t)} \dd t\\
        &= \int_0^\tau(1-t)^{-2}\E\|\nabla \bar{\phi}_{1-t} - \nabla \phi_{1-t}^\star\|^2_{L^2(\msf p_t)}\dd t\,.
    \end{split}
    \end{align}
    {The result then follows from \cref{lem:pointwise}, where we lazily bound the resulting integral:
    \begin{align*}
        \E[\kl{{\msf P}_{[0,\tau]}^\star}{\bar{\msf P}_{[0,\tau]}}] &\leq \frac{R^2 \eps^{-\msf k}}{n}\int_0^\tau (1-t)^{-\msf k - 2} \dd t\\ 
        &\leq \frac{R^2 \eps^{-\msf k}}{n}(1-\tau)^{-\msf k - 2}\,.
    \end{align*}}
\end{proof}

\begin{lemma}[Point-wise drift bound]\label{lem:pointwise}
Under the assumptions of \cref{prop:kl_grisanov_prop}, let $\bar{\phi}_{1-t}$ be the entropic Brenier map between $\bar{\msf p}_t$ and $\bar{\nu}_n$ and $\phi_{1-t}^\star$ be the entropic Brenier map between ${\msf p}_t^\star$ and $\nu$, both with regularization parameter $(1-t)\eps$. Then  
\begin{align*}
    \E\|\nabla \bar{\phi}_{1-t} - \nabla \phi_{1-t}^\star\|^2_{L^2(\msf p_t)} \lesssim \frac{R^2}{n}((1-t)\eps)^{-\msf k}\,.
\end{align*}
\end{lemma}
\begin{proof}
    Setting some notation, we express $\nabla\phi_{1-t}^\star$ as the conditional expectation of the optimal entropic coupling $\pi_t^\star$ between $\msf p_t^\star$ and $\nu$ (recall \cref{prop:selfconsistency}), where we write $\pi_t^\star(z,y) = \gamma_t^\star(z,y) \msf p_t^\star({\rm d}z) \nu({\rm d}y)$. 
    
    The rest of our proof follows a technique due to \cite{stromme2023minimum}: by triangle inequality, we can add and subtract the following term
    \begin{align*}
        \frac{1}{n}\sum_{j=1}^{n}Y_j \gamma_t^\star(z,Y_j)\,,
    \end{align*}
    into the integrand in \eqref{eq:final_integrand}, resulting in
    \begin{align}\label{eq:sumofmaps}
    \begin{split}
       \E\|\nabla \bar{\phi}_{1-t} - \nabla \phi^\star_{1-t}\|^2_{L^2(\msf p_t^\star)} &\lesssim \E\|\nabla \bar{\phi}_{1-t} - n^{-1}\textstyle\sum_{j=1}^{n}Y_j\gamma_t^\star(\cdot,Y_j)\|^2_{L^2(\msf p_t^\star)} \\
       &\qquad + \E\|n^{-1}\textstyle\sum_{j=1}^{n}Y_j \gamma_t^\star(\cdot,Y_j) - \nabla {\phi}_{1-t}^\star\|^2_{L^2(\msf p_t^\star)}\,.
    \end{split}
    \end{align}
    For the second term, with the same manipulations as \citet[Lemma 20]{stromme2023minimum}, we obtain a final bound of
    \begin{align*}
        \E\|n^{-1}{\textstyle\sum_{j=1}^{n}}Y_j \gamma_t^\star(\cdot,Y_j) - \nabla {\phi}_{1-t}^\star\|^2_{L^2(\msf p_t^\star)} = \frac{R^2}{n}\|\gamma_t^\star\|^2_{L^2(\msf p_t^\star \otimes \nu)} \leq \frac{R^2}{n}((1-t)\eps)^{-\msf k}\,,
    \end{align*}
    where the final inequality is also due to \citet[Lemma 16]{stromme2023minimum}. To control the first term in \eqref{eq:sumofmaps}, we also appeal to his calculations of the same theorem: observing that, from \eqref{eq:barbt}
    \begin{align*}
        \nabla\bar{\phi}_{1-t}(z) &= \frac{1}{n}\sum_{j=1}^n Y_j \frac{\exp((g^\star(Y_j) - \tfrac{1}{2(1-t)}\|z - Y_j\|^2)/\eps)}{\frac{1}{n}\sum_{j=1}^n\exp((g^\star(Y_j) - \tfrac{1}{2(1-t)}\|z - Y_j\|^2)/\eps)}\\
        &= \frac{1}{n}\sum_{j=1}^n Y_j\bar{\gamma}_t(z,Y_j)\,.
    \end{align*}
    Since the following equality is true
    \begin{align*}
        \bar{\gamma}_t(z,Y_j) = \frac{\gamma_t^\star(z,Y_j)}{\frac{1}{n}\sum_{k=1}^n\gamma_t^\star(z,Y_{k})}\,,
    \end{align*}
    we can verbatim apply the remaining arguments of \citet[Lemma 20]{stromme2023minimum}. Indeed, for fixed $x \in \R^d$, we have
    \begin{align*}
        \|n^{-1}{\textstyle\sum_{j=1}^{n}}Y_j(\gamma_t^\star(x,Y_j) - \bar{\gamma}_t(x,Y_j))\|^2 \leq R^2\bigl|{\textstyle\sum_{j=1}^{n}}\gamma_t^\star(x,Y_j) - 1  \bigr|^2\,.
    \end{align*}
    Taking the $L^2(\msf p_t^\star)$ norm and the outer expectation, we see that the remaining term is nothing but the first component of the gradient of the dual entropic objective function (see \cref{prop:entropic_dual_grad}), which can be bounded via \cref{lem:dual_gradient_bound}, resulting in the chain of inequalities
    \begin{align*}
        \E\|n^{-1}{\textstyle\sum_{j=1}^{n}}Y_j(\gamma_t^\star(\cdot,Y_j) - \bar{\gamma}_t(\cdot,Y_j))\|^2_{L^2(\msf p_t^\star)} \lesssim \frac{R^2}{n}\|\gamma_t^\star\|^2_{L^2(\msf p_t^\star \otimes \nu)} \leq \frac{R^2}{n}((1-t)\eps)^{-\msf k}\,, 
    \end{align*}
    where the last inequality again holds via \citet[][Lemma 16]{stromme2023minimum}.

\end{proof}

\subsection{Completing the results}

\begin{proof}[Proof of \cref{prop:disc_error}]
This proof closely follows the ideas of  \cite{chen2022sampling}. Applying Girsanov's theorem, we obtain
\begin{align*}
    \tvsq{\hat{\msf{P}}_{[0,\tau]}}{\tilde{\msf{P}}_{[0,\tau]}} &\lesssim \kl{\tilde{\msf{P}}_{[0,\tau]}}{\hat{\msf{P}}_{[0,\tau]}} = \sum_{k=0}^{N-1}\int_{k\eta}^{(k+1)\eta}\E_{\tilde{\msf{P}}_{[0,\tau]}}\| \hat{b}_{k\eta}(X_{k\eta}) - \hat{b}_{t}(X_t)\|^2 \dd t\,.
\end{align*}
Recall that $\eta \in (0,1)$ is a chosen step-size based on $N$, the number of steps to be taken. As in prior analyses, we hope to uniformly bound the integrand above for any $t \in [k\eta,(k+1)\eta]$. Adding and subtracting the appropriate terms, we have
\begin{align}\label{eq:error_decomp}
\begin{split}
    \E_{\tilde{\msf{P}}_{[0,\tau]}}\| \hat{b}_{k\eta}(X_{k\eta}) - \hat{b}_{t}(X_t)\|^2 &\lesssim \E_{\tilde{\msf{P}}_{[0,\tau]}}\|\hat{b}_{k\eta}(X_{k\eta}) -  \hat{b}_{t}(X_{k\eta})\|^2 \\
    &\qquad + \E_{\tilde{\msf{P}}_{[0,\tau]}}\| \hat{b}_{t}(X_{k\eta}) - \hat{b}_{t}(X_{t})\|^2\,. 
\end{split}
\end{align} 
By the semigroup property, we first notice that
\begin{align*}
    \cH_{1-k\eta}[e^{\hat{g}/\eps}\nu_n] = \cH_{t-k\eta}[\cH_{1-t}[e^{\hat{g}/\eps}\nu_n]]\,.
\end{align*}
We can verbatim apply Lemma 16 of \citet{chen2022sampling} with $\bm{q} \defeq  \cH_{1-t}[e^{\hat{g}/\eps}\nu_n]$, $\bm{M}_0 = \text{id}$ and $\bm{M}_1 = (t-k\eta)I$, since $\cH_{1-k\eta}[e^{\hat{g}/\eps}\nu_n] = \bm{q} * \cN(0,(t-k\eta)I)$. This gives
\begin{align*}
    \|\hat{b}_{k\eta}(X_{k\eta}) -  \hat{b}_{t}(X_{k\eta})\|^2 &= \Bigl\|\eps \nabla \log \frac{\bm{q} * \cN(0,(t-k\eta)I)}{\bm{q}}(X_{k\eta})\Bigr\|^2 \\
    &\lesssim L_t^2 \eta d + L_t^2\eta^2\|\eps\nabla \log \bm{q}(X_{kh})\|^2\,.
\end{align*}
Since $\eps \log \bm{q}$ is $L_t$-smooth, we obtain the bounds
\begin{align*}
    \E_{\tilde{\msf{P}}_{[0,\tau]}}\|\eps\nabla \log \bm{q}(X_{kh})\|^2 &\lesssim \E_{\tilde{\msf{P}}_{[0,\tau]}}\|\eps\nabla \log \bm{q}(X_{t})\|^2 + L_t^2\|X_t - X_{kh}\|^2 \\
    &\leq \eps L_td + L_t^2\E_{\tilde{\msf{P}}_{[0,\tau]}}\|X_t - X_{kh}\|^2\,.
\end{align*}
where the final inequality is a standard smoothness inequality (see \cref{lem:smoothness_lemma}). Similarly, the second term on the right-hand side of \eqref{eq:error_decomp} can be bounded by 
\begin{align*}
    \E_{\tilde{\msf{P}}_{[0,\tau]}}\| \hat{b}_{t}(X_{k\eta}) - \hat{b}_{t}(X_{t})\|^2 \leq L_t^2\E_{\tilde{\msf{P}}_{[0,\tau]}}\|X_{k\eta} - X_t\|^2.
\end{align*}
Combining the terms, we obtain
\begin{align*}
    \E_{\tilde{\msf{P}}_{[0,\tau]}}\| \hat{b}_{k\eta}(X_{k\eta}) - \hat{b}_{t}(X_t)\|^2 \lesssim \eps L_t^2 \eta d + L_t^2 \E_{\tilde{\msf{P}}_{[0,\tau]}}\|X_{k\eta} - X_t\|^2\,,
\end{align*}
where, to simplify, we use the fact that $\eta \leq 1/L_t$ (with $L_t \geq 1$), and that $\eta^2 \leq \eta$ for $\eta \in [0,1]$. We now bound the remaining expectation. Under $\tilde{\msf{P}}_{[0,\tau]}$, we can write
\begin{align*}
    X_{t} = \int_0^{t}\hat{b}_s(X_s)\dd s + \sqrt{\eps}B_{t}\,, X_{kh} = \int_0^{k\eta}\hat{b}_s(X_s)\dd s + \sqrt{\eps}B_{k\eta}\,,
\end{align*}
and thus
\begin{align*}
    X_t - X_{k\eta} = \int_{k\eta}^t \hat{b}_s(X_s)\dd s + \sqrt{\eps}(B_t - B_{k\eta})\,.
\end{align*}
Taking squared expectations, writing $\delta \defeq t - k\eta \leq \eta$ (recall that $t \in [k\eta,(k+1)\eta)$), we obtain (through an application of the triangle inequality and Jensen's inequality)
\begin{align*}
    \E_{\tilde{\msf{P}}_{[0,\tau]}}\|X_{t} - X_{k\eta}\|^2 &\lesssim \eps\E_{\tilde{\msf{P}}_{[0,\tau]}}\|B_{t} - B_{k\eta}\|^2 + \delta \int_{k\eta}^t \E_{\tilde{\msf{P}}_{[0,\tau]}}\|\hat{b}_s(X_s)\|^2 \dd s \\
    &\lesssim \eps \eta d + \delta^2 L_td \\
    &\leq (\eps + 1) \eta d 
\end{align*}
where we again used \cref{lem:smoothness_lemma}. Combining all like terms, we obtain the final result.

The estimates for the Lipschitz constant follow from \cref{lem:hessian_lemma}.
\end{proof}

\subsection{Proofs for Section \ref{sec:follmer_sampling}}\label{sec:follmer_calcs}
\subsubsection{Computing Equation \ref{eq:estdrift_follmer}}
The F{\"o}llmer drift is a special case of the Schr{\"o}dinger bridge, where $\mu = \delta_a$ for any $a \in \R^d$. Let $(f^{\msf F},g^{\msf F})$ denote the optimal entropic potentials in this setting. Note that they these potentials are defined up to translation (i.e., the solution is the same if we take $f^{\msf F} + c$ and $g^{\msf F} - c$ for any $c \in \R$). So, we further impose the condition that $f^{\msf F}(a) = 0 = c$. Then the optimality conditions yield
\begin{align}\label{eq:g_follmer}
    g^{\msf F}(y) = \frac{1}{2\eps}\|y\|^2\,.
\end{align}
Plugging this into the expression for the Schr{\"o}dinger bridge drift, we obtain
\begin{align*}
    b^{\msf F}_t(z) &= \eps \nabla \log \cH_{(1-t)\eps}[e^{\tfrac{1}{2\eps}\|\cdot\|^2}\nu](z) \\
    &= (1-t)^{-1}\Bigl(-z + \frac{\int y e^{\tfrac{1}{2\eps}\|y\|^2 - \tfrac{1}{2(1-t)\eps}\|z-y\|^2}\nu({\rm d}y)}{\int e^{\tfrac{1}{2\eps}\|y\|^2 - \tfrac{1}{2(1-t)\eps}\|z-y\|^2}\nu({\rm d}y)} \Bigr)\,.
\end{align*}
Replacing the integrals with respect to $\nu$ with their empirical counterparts yields the estimator.
\subsubsection{Proof of Proposition \ref{cor:sampling_follmer_prop}}
Our goal is to prove the following lemma.
\begin{lemma}
Let $\msf p_\tau$ be the F{\"o}llmer bridge at time $\tau \in [0,1)$ between $\mu = \delta_0$ and $\nu \in \cP_2(\R^d)$ with $\eps = 1$ and suppose the squared second moment of $\nu$ is bounded above by $d$. Then
\begin{align*}
    W_2^2(\msf p_\tau,\nu) \leq d(1-\tau)\,.
\end{align*}
\end{lemma}
\begin{proof}
Note that $\msf p_\tau = \msf P_{1-\tau}$, where $\msf P_{1-\tau}$ is the \emph{reverse} bridge, which starts at $\nu$ and ends at $\mu = \delta_0$.
This reverse bridge is well known to satisfy a simple SDE~\cite{Fol85}: the measure $\msf P_{1-\tau}$ is the law of $Y_{1-\tau}$, where $Y_s$ solves
\begin{equation*}
    {\rm d}Y_s = - \frac{Y_s}{1-s} {\rm d}s + {\rm d} B_s, \quad \quad Y_0 \sim \nu,
\end{equation*}
which has the explicit solution
\begin{equation*}
    Y_s = (1-s)Y_0 + (1-s) \int_0^s \frac{1}{1-r} {\rm d}B_r\,.
\end{equation*}

In particular, we obtain
\begin{align*}
     W_2^2({\msf P}_{s}, \nu) & \leq \E \|Y_s - Y_0\|^2 \\
     & = \E \left\|-s Y_0 + (1-s) \int_0^s \frac{1}{1-r} {\rm d}B_r\right\|^2 \\
     & = s^2 \E \|Y_0\|^2 + d s(1-s) \\
     & \leq d s\,,
\end{align*}
which proves the claim.
\end{proof}

\section{Technical lemmas}
\begin{lemma}[Hessian calculation and bounds]\label{lem:hessian_lemma}
Let $(\msf p_t,b_t)$ be the optimal density-drift pair satisfying the Fokker--Planck equation \eqref{eq:dyneot2} between $\mu_0$ and $\mu_1$. For $t \in [0,1)$, $b_{t}$ is Lipschitz with constant $L_{t}$ given by
\begin{align*}
    L_{t}\defeq \sup_x \| \nabla b_{t}(x)\|_{\mathrm{op}} \leq \frac{1}{(1-t)} \Bigl(1 \vee \| \nabla^2\phi_{1-t}(x) \|_{\mathrm{op}}\Bigr)\,,
\end{align*}
where  $\nabla\phi_{1-t}$ is the entropic Brenier map between $\msf p_t$ and $\mu_1$ with regularization parameter $(1-t)\eps$. Moreover, if the support of $\mu_1$ is contained in $B(0,R)$, then 
\begin{align}
    L_t \leq (1-t)^{-1}(1 \vee R^2((1-t)\eps)^{-1})\,.
\end{align}
\end{lemma}
\begin{proof}
Taking the Jacobian of $b_{t}$, we arrive at
\begin{align*}
    \nabla b_{t}(x) = (1-t)^{-1}(\nabla^2\phi_{1-t}(x) - I)\,,
\end{align*}
As entropic Brenier potentials are convex (recall that their Hessians are covariance matrices; see \eqref{eq:enthessians}), we have the bounds
\begin{align*}
    -(1-t)^{-1}I \preceq \nabla b_{t}(x) \preceq (1-t)^{-1}\nabla^2\phi_{1-t}(x)\,.
\end{align*}
The first claim follows by considering the larger of the two operator norms of both sides.

The second claim follows from the fact that since $\phi_{1-t}$ is an optimal entropic Brenier potential, its Hessian is the conditional covariance of an optimal entropic coupling $\pi_t \in \Gamma(\msf p_t,\mu_1)$, so 
\begin{align*}
    \|\nabla^2\phi_{1-t}(z)\|_{\mathrm{op}} = \frac{1}{(1-t)\eps}\|\text{Cov}_{\pi_t}[Y|X_t=z]\|_{\mathrm{op}} \leq \frac{R^2}{(1-t)\eps}\,,
\end{align*}
since $\supp(\mu_1)\sse B(0,R)$.
\end{proof}

\begin{lemma}\label{lem:smoothness_lemma}
Let $(\msf p_t,b_t)$ be the optimal density-drift pair satisfying the Fokker--Planck equation \eqref{eq:dyneot2} between $\mu_0$ and $\mu_1$. Then for any $t \in [0,1)$
\begin{align*}
    \E_{\msf p_t}\|b_t\|^2 \leq \frac{\eps}{2} L_t d\,.
\end{align*}
\end{lemma}
\begin{proof}
This proof follows the ideas of \citet[Lemma 9]{vempala2019rapid}. We note that the generator given by the forward Schr{\"o}dinger bridge with volatility $\eps$ is 
\begin{align*}
    \mathcal{L}_t f = \frac{\eps}{2}\Delta f - \langle b_t, \nabla f\rangle\,,
\end{align*}
for a smooth function $f$. Writing $b_t = \nabla(\eps \log \cH_{1-t}[e^{g/\eps}\mu_1])$, we obtain
\begin{align*}
    0 = \E_{\msf p_t}\mathcal{L}_t( \eps \log \cH_{1-t}[e^{g/\eps}\mu_1]) \implies \E_{\msf p_t}\|b_t(X_t)\|^2 = \frac{\eps}{2} \E_{\msf p_t}[ \nabla \cdot b_t] \leq \frac{\eps}{2}L_t d\,.
\end{align*}
\end{proof}

\begin{lemma}\citep[][Proposition 3.1]{stromme2023minimum}\label{prop:entropic_dual_grad}
Let $P, Q$ be probability measures on $\R^d$.
For every pair $h_1 = (f_1, g_1) \in L^{\infty}(P)
    \times L^{\infty}(Q)$, there exists an 
    element of $L^{\infty}(P) \times L^{\infty}(Q)$
    which we denote
    by $\nabla \Phi_\eps^{PQ}(f_1, g_1)$ such that
    for all $h_0 = (f_0, g_0) \in L^{\infty}(P) \times L^{\infty}(Q)$,
    \begin{align*}
    \langle \nabla \Phi_\eps^{PQ}(h_1), h_0 \rangle_{L^2(P)\times L^2(Q)}
    &= \int f_0(x) \Big(1 - \int e^{-\eps^{-1}(c(x, y) - f_1(x) - g_1(y))}
    \dd Q(y) \Big) \dd P(x) \\
    &+
     \int g_0(y) \Big(1 - \int e^{-\eps^{-1}(c(x, y) - f_1(x) - g_1(y))}
    \dd P(x) \Big) \dd Q(x).
    \end{align*}
    In other words, the gradient of $\Phi_\eps^{PQ}$
    at $(f_1, g_1)$ is
    the marginal error corresponding to $(f_1, g_1)$.
\end{lemma}

\begin{lemma}\label{lem:dual_gradient_bound}
Following \cref{prop:entropic_dual_grad}, suppose $P = \mu$ and $Q = \nu_n$, where $\nu_n$ is the empirical measure of some measure $\nu$ on the basis of $n$ i.i.d.~samples. Let $(f,g)$ be the optimal entropic potentials between $\mu$ and $\nu$, which induce an optimal entropic coupling $\pi$ (recall \eqref{eq:primal_dual}). Then
\begin{align*}
    \E\|\nabla\Phi^{\mu\nu_n}(f,g)\|^2_{L^2(\mu)\times L^2(\nu_n)} \lesssim \frac{\|\gamma\|_{L^2(\mu\otimes \nu)}^2}{n}\,,
\end{align*}
where the expectation is with respect to the data, and $\gamma = \frac{\dd \pi}{\dd (\mu \otimes \nu)}$.
\end{lemma}
\begin{proof}
Writing out the squared-norm of the gradient explicitly in the norm $L^2(\mu) \times L^2(\nu_n)$, we obtain
\begin{align*}
    \E\|\nabla\Phi^{\mu\nu_n}(f,g)\|^2_{L^2(\mu) \times L^2(\nu_n)} &= \E \int \Bigl(\frac{1}{n}\sum_{j=1}^n \gamma(x,Y_j) - 1 \Bigr)^2 \mu({\rm d}x) \\
    &\qquad + \E \frac{1}{n} \sum_{j=1}^n \Bigl(\int \gamma(x,Y_j) \mu({\rm d}x) - 1 \Bigr)^2\,.
\end{align*}
Note that by the optimality conditions, $\int \gamma(x,Y_j) \mu({\rm d}x) = 1$ for all $Y_j$. Thus, writing $Z_j \defeq \gamma(x,Y_j)$ which are i.i.d., we see that 
\begin{align*}
    \E \int \Bigl(\frac{1}{n}\sum_{j=1}^n \gamma(x,Y_j) - 1 \Bigr)^2 \mu({\rm d}x) &= \int \E\Bigl(\frac{1}{n}\sum_{j=1}^n(Z_j - \E[Z_j])\Bigr)^2 \\
    &= \text{Var}_{\mu \otimes \nu}\Bigl(\frac{1}{n}\sum_{j=1}^n Z_j\Bigr)\\
    &= \frac{1}{n}\text{Var}_{\mu\otimes \nu}(Z_1)\,.
\end{align*}
The remaining component of the squared gradient vanishes, and we obtain 
\begin{align*}
    \E\|\nabla\Phi^{\mu\nu_n}(f,g)\|^2_{L^2(\mu) \times L^2(\nu_n)} = \frac{1}{n}\text{Var}_{\mu\otimes \nu}(\gamma) \leq \frac{\|\gamma\|_{L^2(\mu\otimes\nu)}^2}{n}\,.
\end{align*}
\end{proof}

\end{document}